\documentclass{article}
\pdfoutput=1

\usepackage[utf8]{inputenc}
\usepackage[margin=30mm]{geometry}
\usepackage{parskip}
\usepackage{etoolbox}

\usepackage{amsmath}
\usepackage{amssymb}
\usepackage{amsthm}

\usepackage{newpxtext,newpxmath}

\usepackage[noend]{algpseudocode}
\usepackage{algorithm}

\usepackage{hyperref}
\usepackage[capitalize]{cleveref}
\usepackage[numbers]{natbib}

\newtheorem{remark}{Remark}[section]
\newtheorem{defn}{Definition}[section]
\newtheorem{lemma}{Lemma}[section]
\newtheorem{theorem}{Theorem}[section]
\newtheorem{problem}{Optimisation Problem}[section]

\usepackage{adjustbox}
\usepackage{pgfplots}
\usepackage{tikz-cd}
\usepackage{tikz}
\usetikzlibrary{shapes.geometric}
\usetikzlibrary{positioning}
\tikzstyle{circ} = [circle, text centered, draw=black]
\pgfplotsset{compat=1.16}

\usepackage{ifthen}
\newboolean{include-notes}
\setboolean{include-notes}{true}
\newcommand{\TODO}[1]{\ifthenelse{\boolean{include-notes}}
 {{\color{red} \textbf{TODO:} #1}}{}}
\newcommand{\Adam}[1]{\ifthenelse{\boolean{include-notes}}
 {{\color{mygreen} \textbf{AG:} #1}}{}}
\newcommand{\Sam}[1]{\ifthenelse{\boolean{include-notes}}
 {{\color{cyan} \textbf{ST:} #1}}{}}

\newcommand\labelAndRemember[2]
  {\expandafter\gdef\csname labeled:#1\endcsname{#2}%
   \label{#1}#2}
\newcommand\recallLabel[1]
   {\csname labeled:#1\endcsname\tag{\ref{#1}}}

\newcounter{termcounter}

\makeatletter
\if@cref@capitalise\crefname{term}{Term}{Terms}\else\crefname{term}{term}{terms}\fi
\creflabelformat{term}{#2{(#1)}#3}
\def\term{\@ifnextchar[\term@optarg\term@noarg}%
\def\term@optarg[#1]#2{%
  {#1}%
  \def\@currentlabel{#1}%
  \def\cref@currentlabel{[][2147483647][]#1}%
  \cref@label[term]{#2}}
\def\term@noarg#1{%
  \refstepcounter{termcounter}%
  {(\alph{termcounter})}%
  \cref@label[term]{#1}}
\makeatother

\newcommand{\expectation}{\mathbb{E}}

\newcommand{\soft}{\mathrm{soft}}
\newcommand{\qsoft}{Q^{\soft}}
\newcommand{\vsoft}{V^{\soft}}
\newcommand{\asoft}{A^{\soft}}

\newcommand{\loglikelihood}{\mathcal{L}}
\newcommand{\loss}{L}
\newcommand{\trajectory}{\tau}
\newcommand{\expertdataset}{\mathcal{D}}
\newcommand{\discriminator}{D}
\newcommand{\discriminatorloss}{L}

\DeclareMathOperator*{\argmax}{arg\,max}

\DeclareMathOperator*{\expect}{{\mathbb E}}

\newcommand{\statespace}{\mathcal{S}}
\newcommand{\actionspace}{\mathcal{A}}
\newcommand{\transitiondist}{\mathcal{T}}
\newcommand{\discount}{\gamma}
\newcommand{\horizon}{T}
\newcommand{\initialstatedist}{\mathcal I}
\newcommand{\reward}{r}
\newcommand{\mdp}{(\statespace, \actionspace, \discount, \horizon, \initialstatedist, \transitiondist, \reward)}

\newcommand{\state}{s}
\newcommand{\action}{a}
\newcommand{\nextstate}{s'}

\newcommand{\policyreturn}{G}
\newcommand{\policy}{\pi}

\makeatletter
\newcommand{\vast}{\bBigg@{3}}
\newcommand{\Vast}{\bBigg@{4}}
\makeatother

\newcommand\giv[1][]{\:#1\vert\:}

\title{A Primer on Maximum Causal Entropy Inverse Reinforcement Learning}
\author{
  Adam Gleave\thanks{Equal contribution.}\\\href{mailto:Adam Gleave <gleave@berkeley.edu>?subject=MCE IRL}{\texttt{gleave@berkeley.edu}}
\and
  Sam Toyer\footnotemark[1]\\\href{mailto:Sam Toyer <sdt@berkeley.edu>?subject=MCE IRL}{\texttt{sdt@berkeley.edu}}
}
\date{}

\makeatletter
\patchcmd{\@maketitle}{\vskip 2em}{\vspace*{-1em}}{}{}  %
\makeatother

\begin{document}

\maketitle

\begin{abstract}
\emph{Inverse Reinforcement Learning} (IRL) algorithms~\cite{ng2000algorithms,abbeel2004apprenticeship} infer a reward function that explains demonstrations provided by an expert acting in the environment.
Maximum Causal Entropy (MCE) IRL~\cite{ziebart2010paper,ziebart2010modeling} is currently the most popular formulation of IRL, with numerous extensions~\cite{finn2016connection,fu2017learning,shah2019preferences}.
In this tutorial, we present a compressed derivation of MCE IRL and the key results from contemporary implementations of MCE IRL algorithms.
We hope this will serve both as an introductory resource for those new to the field, and as a concise reference for those already familiar with these topics.
\end{abstract}

\tableofcontents

\newpage

\section{Introduction}
The most direct approach to automating a task is to manually specify the steps required to complete the task in the form of a \emph{policy}.
However, it is often easier to specify a \emph{reward function} that captures the overarching task objective, and then use reinforcement learning (RL) to obtain a policy that carries out the steps to achieve that objective~\cite{sutton2018rl}.
Unfortunately, procedurally specifying a reward function can also be challenging.
Even a task as simple as peg insertion from pixels has a non-trivial reward function~\citep[Section~IV.A]{vecerik2019peg}.
Most real-world tasks have far more complex reward functions than this, especially when they involve human interaction.

A natural solution is to learn the reward function itself.
A common approach is \emph{Inverse Reinforcement Learning} (IRL)~\cite{ng2000algorithms,abbeel2004apprenticeship}: inferring a reward function from a set of demonstrations of a particular task.
This is well-suited to tasks that humans can easily perform but would find difficult to directly specify the reward function for, such as walking or driving.
An additional benefit is that demonstrations can be cheaply collected at scale: for example, vehicle manufacturers can learn from their customers' driving behaviour~\cite{tesla2016fleet,karpathy2020keynote}.

A key challenge for IRL is that the problem is \emph{underconstrained}: many different reward functions are consistent with the observed expert behaviour~\citep{amin2017,kim2021,cao2021,skalse2022invariance}.
Some of these differences, such as scale or potential shaping, \emph{never} change the optimal policy and so may be ignored~\cite{ng1999policy,gleave2020quantifying}.
However, many differences \emph{do} change the optimal policy---yet perhaps only in states that were never observed in the expert demonstrations.
By contrast, alternative modalities such as actively querying the user for preference comparisons~\cite{sadigh2017active} can avoid this ambiguity, at the cost of a higher cognitive workload for the user.

\emph{Maximum Causal Entropy} (MCE) IRL is a popular framework for IRL.
Introduced by Ziebart~\cite{ziebart2010paper,ziebart2010modeling}, MCE IRL models the demonstrator as maximising return achieved (like an RL algorithm) plus an entropy bonus that rewards randomising between actions.
The entropy bonus allows the algorithm to model suboptimal demonstrator actions as random mistakes.
In particular, it means that \textit{any} set of sampled demonstrations has support under the demonstrator model, even if the trajectories are not perfectly optimal for any non-trivial reward function.
This is important for modelling humans, who frequently deviate from optimality in complex tasks.

An alternative framework, Bayesian IRL~\citep{ramachandran2007birl}, goes beyond finding a point estimate of the reward function and instead infers a \emph{posterior distribution} over reward functions.
It therefore assigns probability mass to \emph{all} reward functions compatible with the demonstrations (so long as they have support in the prior).
Unfortunately, Bayesian IRL is difficult to scale, and has to date only been demonstrated in relatively simple environments such as small, discrete MDPs.

In contrast to Bayesian IRL, algorithms based on MCE IRL have scaled to high-dimensional environments.
Maximum Entropy Deep IRL~\cite{wulfmeier2015maximum} was one of the first extensions, and is able to learn rewards in gridworlds from pixel observations.
More recently, Guided Cost Learning~\cite{finn2016guided} and Adversarial IRL~\cite{fu2017learning} have scaled to MuJoCo continuous control tasks.
Given its popularity and accomplishments we focus on MCE IRL in the remainder of this document; we refer the reader to \citet{jeon2020rrc} for a broader overview of reward learning.

\section{Background}

Before describing the MCE IRL algorithm itself, we first need to introduce some notation and concepts. First, we define \emph{Markov Decision Processes} (MDPs). Then, we outline IRL based on \emph{feature matching}. Finally, we introduce the notion of \emph{causal entropy}.

\subsection{Markov decision processes}

In this tutorial, we consider Markov decision processes that are either discounted, or have a finite horizon, or both. Below we give the definition and notation that we use throughout. Note that when doing IRL, we drop the assumption that the true reward $\reward$ is known, and instead infer it from data.

\begin{defn}
A \emph{Markov Decision Process (MDP)} $M = \mdp$ consists of
a set of states $\statespace$ and a set of actions $\actionspace$;
a discount factor $\discount \in [0,1]$;
a horizon $T \in \mathbb{N} \cup \{\infty\}$;
an initial state distribution $\initialstatedist(\state)$;
a transition distribution $\transitiondist(\nextstate \giv \state,\action)$ specifying the probability of transitioning to $\nextstate{}$ from $\state{}$ after taking action $\action{}$;
and a reward function $\reward(\state, \action, \nextstate)$ specifying the reward upon taking action $\action$ in state $\state$ and transitioning to state $\nextstate$.
It must be true that either the discount factor satisfies $\discount < 1$ or that the horizon is finite ($T < \infty$).
\end{defn}

Given an MDP, we can define a (stochastic) \emph{policy} $\policy_t(\action_t \giv \state_t)$ that assigns a probability to taking action $\action_t \in \actionspace$ in state $\state_t \in \statespace$ at time step $t$.
The probability of a policy acting in the MDP producing a \textit{trajectory fragment} $\trajectory{} = \left(s_0, a_0, s_1, a_1, \ldots, s_{k-1}, a_{k-1}, s_k\right)$ of length $k \in \mathbb{N}$ is given by:
\begin{equation}
\label{eq:trajectory-probability}
p(\trajectory{}) = \initialstatedist(s_0) \prod_{t=0}^{k-1} \transitiondist{}(s_{t+1} \giv s_{t}, a_{t})~\policy{}_t(a_{t} \giv s_t)\,.
\end{equation}
Note in a finite-horizon MDP ($T \in \mathbb{N}$) the policy may be \textit{non-stationary}, i.e.\ it can depend on the time step  $t$.
In the infinite-horizon case ($T = \infty$), the MDP is symmetric over time, and so we assume the policy is stationary.
We drop the subscript and write only $\policy$ when the policy is acting across multiple timesteps or is known to be stationary.

The objective of an agent is to maximise the expected return:
\begin{equation}
\policyreturn(\policy) = \expect_{\policy{}}\left[\sum_{t=0}^{T-1} \discount^t \reward{}(S_t, A_t, S_{t+1})\right]\,.
\end{equation}
An optimal policy $\policy^*$ attains the highest possible expected return: $\policy^* \in \argmax_{\policy} \policyreturn(\policy)$.

\subsection{Imitation as feature expectation matching}

In IRL, our objective is to recover a reward function that---when maximised by a reinforcement learner---will lead to similar behaviour to the demonstrations.
One way to formalise ``similar behaviour'' is by \textit{feature expectation matching}.
Suppose the demonstrator is optimising some unknown linear reward function $r_{\theta_\star}(s_t, a_t) = \theta_\star^T \phi(s_t, a_t)$, where $\phi(s_t, a_t) \in \mathbb R^d$ is some fixed feature mapping.
In this case, the demonstrator's expected return under its behaviour distribution $\expertdataset{}$ will be linear in the expected sum of discounted feature vectors observed by the agent:
\begin{equation}
\expect_{\mathcal \expertdataset{}}\left[\sum_{t=0}^{T-1} \gamma^t r_{\theta_\star}(S_t, A_t)\right]
= \expect_{\expertdataset{}}\left[\sum_{t=0}^{T-1} \gamma^t \theta_\star^T \phi(S_t, A_t) \right]
= \theta_\star^T \expect_{\expertdataset{}}\left[\sum_{t=0}^{T-1} \gamma^t \phi(S_t, A_t) \right]\,.
\end{equation}

Say that we recover some imitation policy $\policy{}$ with identical expected feature counts to the demonstrator:
\begin{equation}
\expect_{\policy{}}\left[\sum_{t=0}^{T-1} \gamma^t \phi(S_t, A_t) \right]
= \expect_{\expertdataset{}}\left[\sum_{t=0}^{T-1} \gamma^{t} \phi(S_t, A_t) \right]\,.
\end{equation}
Because expected reward is linear in the (matched) feature expectations above, the reward obtained by $\policy{}$ under the unknown true reward function $r_{\theta_\star}(s_t, a_t)$ must be the same as the reward obtained by the demonstrator $\expertdataset{}$ under that reward function~\cite{abbeel2004apprenticeship}.
If our imitation policy $\policy{}$ is optimal under reward function parameters $\hat{\theta}$, then it is reasonable to say that $\hat{\theta}$ is an estimate of the demonstrator's true reward parameters.
However, in general there will be many choices of $\hat{\theta}$ that produce imitation policies $\policy{}$ with the same expected feature counts as the demonstrator.
In the next section, we will see how we can apply the \emph{principle of maximum entropy} to break ties between these reward functions.

\subsection{Maximum causal entropy}

The \emph{principle of maximum entropy} holds that when choosing between many probability distributions that are consistent with the data, one should pick the distribution that has \textit{highest entropy}.
Intuitively, such a distribution is the ``most uncertain'' among those that meet the data-fitting constraints.
This principle can also be formally justified with an appeal to games: choosing the maximum entropy distribution minimises one's expected log loss in the setting where an adversary is able to choose the true distribution from those consistent with the data~\cite{topsoe1979information}.

In an IRL setting, this principle leads to a simple and effective algorithm for simultaneously recovering a reward function and corresponding imitation policy from demonstrations.
In particular, in MCE IRL we choose the reward function whose corresponding policy has maximal \emph{causal entropy}:

\begin{defn}\label{defn:causal-ent}
Let $S_{0:T}$ and $A_{0:T}$ be random variables representing states and actions induced by following a policy $\policy$ in an MDP and sampled according to \cref{eq:trajectory-probability}.
Then the causal entropy\footnote{The definition of causal entropy can also be generalised to non-Markovian contexts~\cite[section~4.2]{ziebart2010modeling}.} $H(A_{0:T-1} \| S_{0:T-1})$ is the sum of the entropies of the policy action selection conditioned on the state at that timestep:
\begin{align}
    H(A_{0:T-1} \| S_{0:T-1}) &= \expect_{\policy{}}\left[-\sum_{t=0}^{T-1} \discount^t \log \policy{}_t(A_t \giv S_t)\right]
    = \sum_{t=0}^{T-1} \discount^t H(A_{t} \giv S_t)\,.
\end{align}
\end{defn}

Note the sum in \cref{defn:causal-ent} is discounted, effectively valuing entropy of later actions less.
This is needed for consistency with discounted returns, and for convergence in infinite-horizon problems; see \citet[appendix~A]{haarnoja2017energy} for more information.
We will revisit the subtleties of infinite horizons in \cref{ssec:inf-horizons-suck}.

\begin{remark}
The causal entropy $H(A_{0:T-1} \| S_{0:T-1})$ has the useful property that at each time step, it conditions only on information available to the agent (that is, the current state, as well as prior states and actions).
By contrast, the conditional entropy of actions given states $H(A_{0:T-1} \giv S_{0:T-1})$ conditions on states that arise \emph{after} each action was taken.
Moreover, conventional Shannon entropy $H(S_{0:T-1}, A_{0:T-1})$ calculates the entropy over the entire trajectory distribution, introducing an unwanted dependency on transition dynamics via terms $H(S_t \giv S_{t-1}, A_{t-1})$:
\begin{align}
H(S_{0:T-1}, A_{0:T-1}) &= \sum_{t=0}^{T-1} \discount^t H(S_t, A_t \giv S_{0:t-1}, A_{0:t-1}) & & \text{chain rule} \\
&= \sum_{t=0}^{T-1} \discount^t H(S_t \giv S_{0:t-1}, A_{0:t-1}) + \sum_{t=0}^{T-1} \discount^t H(A_t \giv S_{0:t}, A_{0:t-1}) & & \text{chain rule} \\
&= \sum_{t=0}^{T-1} \discount^t H(S_t \giv S_{t-1}, A_{t-1}) + \sum_{t=0}^{T-1} \discount^t H(A_t \giv S_{t}) & & \text{independence} \\
&= \sum_{t=0}^{T-1} \underbrace{\discount^t H(S_t \giv S_{t-1}, A_{t-1})}_{\text{state transition entropy}} + \underbrace{H(A_{0:T-1} \| S_{0:T-1})}_{\text{causal entropy}}\, & & \text{causal ent. definition}
\end{align}
Maximising Shannon entropy therefore introduces a bias towards taking actions with uncertain (and possibly risky) outcomes.
For this reason, maximum causal entropy should be used rather than maximum (Shannon) entropy in stochastic MDPs.
In deterministic MDPs, $H(S_t \giv S_{t-1}, A_{t-1}) = 0$ and the methods are equivalent (\cref{sec:mce-irl:max-ent}).
\end{remark}

\section{Maximum Causal Entropy (MCE) IRL}\label{sec:mce-irl}

In this section, we start by reviewing two complementary ways of formalising the MCE IRL objective as an optimisation problem.
First, we will consider MCE IRL as a way of solving a particular constrained optimisation problem.
Second, we will describe MCE IRL in terms of maximum likelihood estimation, which will allow us to replace the linear reward function with a non-linear one.
Finally, we will discuss how MCE IRL simplifies to Maximum Entropy (ME) IRL under deterministic dynamics.

\subsection{MCE IRL as feature expectation matching}\label{sec:mce-irl:fem}

MCE IRL was originally introduced by \citet{ziebart2010paper}, who considered the general setting of non-Markovian dynamics and trajectory-centric features.
They then derived simplifications for the special case of Markovian dynamics and policies, as well as feature functions that decompose across state--action transitions.
This primer only considers the simpler case (decomposed features, Markovian policies), and is thus substantially shorter.
Moreover, we will assume throughout that the horizon $T$ is finite, and that the policy is non-stationary (time-dependent); removing these assumptions is discussed in \cref{ssec:inf-horizons-suck}

\begin{problem}[MCE IRL primal problem]\label{prob:mce-irl-primal-orig}
Given an expert data distribution $\expertdataset$, the optimisation problem of finding a time-dependent stochastic policy $\policy{}_t(s_t, a_t)$ that matches feature expectations while maximising causal entropy can be written as
\begin{align}
  \max_{
    \pi \in \mathscr P
  } \quad & H(A_{0:T-1} \| S_{0:T-1}) \\
  \text{Subject to} \quad
  & \expect_{\policy{}}\left[\sum_{t=0}^{T-1} \discount^t \phi(S_t, A_t)\right] = \expect_{\expertdataset{}}\left[\sum_{t=0}^{T-1} \discount^t \phi(S_t, A_t)\right]
\end{align}
where $\mathscr P$ denotes the set of policies where the action distribution at each state falls within the standard policy simplex, so that
\begin{equation}
    \pi \in \mathscr P \iff \policy{}_t(a_t\giv{}s_t) \geq 0 \text{ and } \sum_{a'} \policy_t(a'\giv{}s_t)=1 \quad \left( \forall  0 \leq t < T, \forall a_t \in \actionspace, \forall s_t \in \statespace \right)~,
\end{equation}
and $S_{0:T-1}$ and $A_{0:T-1}$ are random sequences of states and actions, induced by $\policy{}_t$ and sampled according to \cref{eq:trajectory-probability} on the left, and $\expertdataset{}$ on the right.
\end{problem}

This primal problem is not convex in general, but we will nevertheless approach solving it using a standard recipe for convex optimisation problems~\cite[Chap.~5]{boyd2004convex}:
\begin{enumerate}
    \item \textbf{Form the Lagrangian.} First, we will convert the feature expectation matching constraint into a weighted penalty term, producing the Lagrangian for the problem.
    \item \textbf{Derive the dual problem.} Next, we will use the Lagrangian to form a \textit{dual problem} that is equivalent to \cref{prob:mce-irl-primal-orig}.
    The dual problem will introduce an extra set of parameters---called Lagrange multipliers---which in this case will be interpretable as weights of a reward function.
    \item \textbf{Dual ascent.} Finally, we will solve the dual problem with a simple procedure called \textit{dual ascent}, which alternates between exactly maximising the Lagrangian with respect to the policy parameters (the primal variables) and taking a single gradient step to minimise the Lagrangian with respect to the reward weights (the dual variables).
    This process is repeated until the dual variables converge.
\end{enumerate}

\subsubsection{The Lagrangian and the dual problem}

We begin our derivation of the dual problem by forming the Lagrangian $\Lambda : \mathscr P \times \mathbb R^d \to \mathbb R$ for \cref{prob:mce-irl-primal-orig}:
\begin{subequations}
\begin{align}
  \Lambda(\policy{}, \theta) &= H(A_{0:T-1} \| S_{0:T-1}) \label{eqn:lag-ent}\\
    &\quad + \theta^T \left(
       \expect_{\policy{}}\left[\sum_{t=0}^{T-1} \discount^t \phi(S_t, A_t)\right]
       - \expect_{\expertdataset{}}\left[\sum_{t=0}^{T-1} \discount^t \phi(S_t, A_t)\right]
    \right)\,.\label{eqn:lag-fem}
\end{align}
\end{subequations}
The dual variable vector $\theta \in \mathbb R^d$ can be interpreted as a set of penalty weights to enforce the feature expectation matching constraint in \cref{eqn:lag-fem}.
Importantly, $\policy{}$ is constrained to the simplex $\mathscr P$; later, we will need be careful with this constraint when we form a separate, nested optimisation problem to compute $\pi$.

Equipped with the Lagrangian, we can now express the dual problem to \cref{prob:mce-irl-primal-orig}
\begin{problem}[Dual MCE IRL problem]\label{prob:mce-irl-dual}
Define the dual function $g(\theta)$ as
\begin{equation}
    g(\theta) = \max_{\pi \in \mathscr P} \Lambda(\pi, \theta)\,.
\end{equation}
The dual MCE IRL problem is the problem of finding a $\pi^*$ and $\theta^*$ such that $\theta^*$ attains
\begin{equation}\label{eqn:mce-irl-dual:theta-obj}
    g(\theta^*) = \min_{\theta \in \mathbb R^d} g(\theta)\,,
\end{equation}
while $\pi^*$ attains
\begin{equation}
    \Lambda(\pi^*, \theta^*) = g(\theta^*)\,.
\end{equation}
\end{problem}
Instead of optimising the primal (\cref{prob:mce-irl-primal-orig}), we will instead optimise the dual (\cref{prob:mce-irl-dual}), and treat the recovered $\pi^*$ as a solution for the primal imitation learning problem.
Moreover, we will later see that the recovered $\theta^*$ can be interpreted as parameters of a reward function that incentivises reproducing the expert demonstrations.

\subsubsection{The dual function}

Observe that \cref{prob:mce-irl-dual} is expressed in terms of a dual function $g(\theta) = \max_{\pi} \Lambda(\pi, \theta)$.
Computing $g(\theta)$ can therefore be viewed as a nested optimisation over $\pi$ subject to the probability simplex constraints defining $\mathscr P$.

We will begin by putting the optimisation defining the dual function into the familiar form.
First recall that the policy simplex $\mathscr P$ is defined by two constraints:
\begin{equation}
    \policy{}_t(a_t\giv{}s_t) \geq 0 \text{ and } \sum_{a} \policy_t(a\giv{}s_t)=1 \quad \left( \forall  0 \leq t < T, \forall a_t \in \actionspace, \forall s_t \in \statespace \right)
\end{equation}
Since the causal entropy term in $\Lambda$ is undefined when $\pi$ has negative elements, we will treat the non-negativity constraint as implicit, and focus only on the normalisation constraint.
We will rewrite the normalisation constraint as $h_{s_t}(\pi) = 0$, where the function $h_{s_t}$ is defined as
\begin{equation}
    h_{s_t}(\pi) = \sum_{a \in \mathcal A} \pi(a \giv s_t) - 1\,.
\end{equation}
This gives rise the following optimisation problem:
\begin{problem}[Dual function primal problem]\label{prob:mce-irl-dual-func-primal} The problem of computing the dual function can be equivalently expressed as
\begin{align}
  \max_{
    \pi \in \mathscr P
  } \quad & \Lambda(\pi, \theta) \\
  \text{Subject to} \quad
  & h_{s_t}(\pi) = 0 \quad  \left( \forall  0 \leq t < T, \forall s_t \in \statespace \right)
\end{align}
\end{problem}

The Lagrangian for \cref{prob:mce-irl-dual-func-primal} is
\begin{align}
    \Psi(\pi, \mu; \theta) &= \Lambda(\pi, \theta) + \sum_{s_t\in \mathcal S, 0 \leq t < T} \mu_{s_t} h_{s_t}(\pi) \\
    &= H(A_{0:T-1} \| S_{0:T-1}) + \theta^T \left(
       \expect_{\policy{}}\left[\sum_{t=0}^{T-1} \discount^t \phi(S_t, A_t)\right]
       - \expect_{\expertdataset{}}\left[\sum_{t=0}^{T-1} \discount^t \phi(S_t, A_t)\right]
    \right) \\
    &\qquad + \sum_{s_t\in \mathcal S, 0 \leq t < T} \mu_{s_t} \left(\sum_{a \in \actionspace} \policy{}_t(a \giv s_{t}) - 1 \right)
\end{align}
where $\left\{\mu_{s_t}\right\}_{s_t\in \mathcal S, 0 \leq t < T}$ are dual variables for the normalisation constraints, and we have left the dependence on $\theta$ implicit.
We will attempt to find a minimiser of \cref{prob:mce-irl-dual-func-primal} by solving the Karush-Kuhn-Tucker (KKT) conditions for the problem~\cite[p.224]{boyd2004convex}, which are\footnote{This problem does not have explicit inequality constraints, so we do not require complementary slackness conditions.}
\begin{align}
    \nabla_\pi \Psi(\pi, \mu; \theta) &= 0 \label{eqn:mce-irl-dual-func-kkt-grad-cond}\\
    h_{s_t}(\pi) &= 0 \quad  \left( \forall  0 \leq t < T, \forall s_t \in \statespace \right)\,.\label{eqn:mce-irl-dual-func-kkt-norm-cond}
\end{align}
Notice that $\mu$ only appears in the first equation.
We will later see that we can set it to a value which ensures that the normalisation constraints are satisfied and the gradient vanishes.
First, however, we will compute the gradient of the Lagrangian for \cref{prob:mce-irl-dual-func-primal}:

\begin{lemma}
The gradient of $\Psi(\pi, \mu; \theta)$ is given by
\begin{align}
 \nabla_{\policy{}_t(a_t \giv s_t)} &\Psi(\policy{}, \mu; \theta)
  = \mu_{s_t} - \rho_{\policy,t}(s_t)\vast(1 + \log \policy{}_t(a_t \giv s_t) - \theta^T \phi(s_t, a_t) \nonumber \\
  &\quad - \expect_{\policy{}}\left[\sum_{t'=t+1}^{T-1} \discount^{t' - t}\left(\theta^T \phi(S_{t'}, A_{t'}) - \log \policy{}_{t'}(A_{t'} \giv S_{t'})\right) \giv[\middle]{} s_t, a_t \right]\vast)\,,\label{eqn:mce-irl-lagrangian-grad}
\end{align}
where
\begin{equation}
\rho_{\policy,t}(s_t) = \discount^t\expectation_{\policy}\left[\mathbb I[S_t=s_t]\right] = \discount^t \mathbb P_\pi(S_t=s_t)
\end{equation}
is the discounted probability that the agent will be in state $s_t$ at time $t$ if it follows policy $\policy{}$.
\end{lemma}

\begin{proof}
We will derive this gradient term-by-term.
Taking the derivative of the third term (normalisation) with respect to some arbitrary $\policy{}_t(a_t \giv s_t)$ yields
\begin{align}
  \nabla_{\policy{}_t(a_t \giv s_t)} \sum_{t'=0}^{T-1} \sum_{s_{t'} \in \statespace} \mu_{s_{t'}} \left(\sum_{a' \in \actionspace} \policy{}_t(a' \giv s_{t'}) - 1 \right)
  &= \mu_{s_t}\,.
\end{align}
Note that we are differentiating with respect to the action selection probability $\policy_t(a_t\giv s_t)$, which is a variable specific to time $t$; the policy need not be stationary, so we may have $\policy_t(a \giv s) \neq \policy_{t'}(a\giv{}s)$.

The derivative of the middle term (feature expectation matching) is
\begin{align}
  \nabla_{\policy{}_t(a_t \giv s_t)} &\theta^T \left(
       \expect_{\policy{}}\left[\sum_{t'=0}^{T-1} \discount^{t'} \phi(S_{t'}, A_{t'})\right]
       - \expect_{\expertdataset{}}\left[\sum_{t'=0}^{T-1} \discount^{t'} \phi(S_{t'}, A_{t'})\right]
    \right) \nonumber \\
    &= \nabla_{\policy{}_t(a_t \giv s_t)} \theta^T \expect_{\policy{}}\left[\sum_{t'=0}^{T-1} \discount^{t'} \phi(S_{t'}, A_{t'})\right] \\
    &= \theta^T \expect_{S_t \sim \policy{}}\left[
         \nabla_{\policy{}_t(a_t \giv s_t)} \expect_{\policy{}}\left[
            \sum_{t'=t}^{T-1} \discount^{t'} \phi(S_{t'}, A_{t'}) \giv[\middle]{} S_t
        \right]
    \right] \\
    &= \theta^T \expect_{S_t \sim \policy{}}\left[
         \gamma^t \mathbb I[S_t=s_t] \nabla_{\policy{}_t(a_t \giv s_t)} \expect_{\policy{}}\left[
            \sum_{t'=t}^{T-1} \discount^{t'-t} \phi(S_{t'}, A_{t'}) \giv[\middle]{} S_t = s_t
        \right]
    \right] \\
    &= \rho_{\policy,t}(s_t) \theta^T \nabla_{\policy{}_t(a_t \giv s_t)} \expect_{\policy{}}\left[
        \sum_{t'=t}^{T-1} \discount^{t'-t} \phi(S_{t'}, A_{t'}) \giv[\middle]{} S_t = s_t
    \right] \\
    &= \rho_{\policy,t}(s_t) \theta^T \nabla_{\policy{}_t(a_t \giv s_t)} \left( \vphantom{\left(\sum_{a'_t \neq a_t}\policy{}_t(a'_t \giv s_t)\right) }
        \policy_t(a_t \mid s_t)\expect\left[
            \sum_{t'=t}^{T-1} \discount^{t'-t} \phi(S_{t'}, A_{t'}) \giv[\middle]{} S_t = s_t, A_t = a_t
        \right] \right. \nonumber\\
    &\qquad\qquad\qquad \left. + \left(\sum_{a'_t \neq a_t}\policy{}_t(a'_t \giv s_t)\right) \cdot \expect\left[
            \sum_{t'=t}^{T-1} \discount^{t'-t} \phi(S_{t'}, A_{t'}) \giv[\middle]{} S_t = s_t, A_t \neq a_t
        \right]
    \right) \\
    &= \rho_{\policy,t}(s_t) \theta^T \expect\left[
        \sum_{t'=t}^{T-1} \discount^{t'-t} \phi(S_{t'}, A_{t'}) \giv[\middle]{} s_t, a_t
    \right] \\
    &= \rho_{\policy,t}(s_t) \, \theta^T \left[\phi(s_t, a_t) + \expect_{\policy{}} \left[\sum_{t'=t+1}^{T-1} \discount^{t' - t}\phi(S_{t'}, A_{t'}) \giv[\middle]{} s_t, a_t\right]\right]\,,
\end{align}
Finally, the derivative of the first term (causal entropy) can be derived in a similar manner as
\begin{align}
  &\nabla_{\policy{}_t(a_t \giv s_t)}  {H(A_{0:T-1} \| S_{0:T-1})}
  = -\nabla_{\policy{}_t(a_t \giv s_t)} \expect_{\policy{}}\left[\sum_{t'=0}^{T-1} \discount^t \log \policy{}_{t'}(A_{t'} \giv S_{t'})\right]\\
  &\qquad= -\rho_{\policy,t}(s_t) \nabla_{\policy{}_t(a_t \giv s_t)} \expect_{\policy{}}\left[\sum_{t'=t}^{T-1} \discount^{t'-t} \log \policy{}_{t'}(A_{t'} \giv{} S_{t'}) \giv[\middle]{} s_t \right]\\
  &\qquad= -\rho_{\policy,t}(s_t) \nabla_{\policy{}_t(a_t \giv s_t)} \expect_{\policy{}}\left[ \log \pi_t(A_t \mid S_t) + \sum_{t'=t+1}^{T-1} \discount^{t'-t} \log \policy{}_{t'}(A_{t'} \giv{} S_{t'}) \giv[\middle]{} s_t \right]\\
  &\qquad= -\rho_{\policy,t}(s_t) \left[1 + \log \policy{}_t(a_t \giv s_t) + \expect_{\policy}\left[\sum_{t'=t+1}^{T-1} \discount^{t' - t} \log \policy{}_{t'}(A_{t'}\giv{}S_{t'}) \giv[\middle]{} s_t, a_t \right] \right]\,.\label{eqn:grad-cent-final:XXX}
\end{align}
Putting it all together, the derivative of $\Psi(\pi,\mu;\theta)$ with respect to our policy is
\begin{align*}
 \nabla_{\policy{}_t(a_t \giv s_t)} &\Psi(\policy{}, \mu; \theta)
  = \mu_{s_t} - \rho_{\policy,t}(s_t)\vast(1 + \log \policy{}_t(a_t \giv s_t) - \theta^T \phi(s_t, a_t) \\
  &\quad - \expect_{\policy{}}\left[\sum_{t'=t+1}^{T-1} \discount^{t' - t}\left(\theta^T \phi(S_{t'}, A_{t'}) - \log \policy{}_{t'}(A_{t'} \giv S_{t'})\right) \giv[\middle]{} s_t, a_t \right]\vast)\,. \tag{\ref{eqn:mce-irl-lagrangian-grad}}
\end{align*}
\end{proof}

We will now solve the first KKT condition, \cref{eqn:mce-irl-dual-func-kkt-grad-cond}, by setting \cref{eqn:mce-irl-lagrangian-grad} equal to zero and solving for $\pi$.
This will give us the $\pi$ which attains $g(\theta) = \Lambda(\pi, \theta)$, thereby allowing us to achieve our goal (for this sub-section) of computing the dual function $g(\theta)$.
Conveniently, the resulting $\pi$ has a form which resembles the optimal policy under value iteration---indeed, the recursion defining $\pi$ is often called \textit{soft value iteration}:
\begin{lemma}
The KKT condition $\nabla_\pi \Psi(\pi, \mu; \theta) = 0$ is satisfied by a policy
\begin{equation}\label{eqn:policy-soft-decomposition}
\pi_t(a_t \mid s_t) = \exp\left(\qsoft_{\theta,t}(s_t, a_t) - \vsoft_{\theta,t}(s_t)\right)
\end{equation}
satisfying the following recursion:
\begin{equation}\label{eqn:policy-soft-vi}
\left.\begin{aligned}
  \vsoft_{\theta,t}(s_t) &= \log \sum_{a_t \in \actionspace} \exp \qsoft_{\theta,t}(s_t, a_t) & \ (\forall 0 \leq t \leq T-1)\\
  \qsoft_{\theta,t}(s_t, a_t) &= \theta^T \phi(s_t, a_t) + \discount \expect_{\mathcal T}\left[\vsoft_{\theta, t+1}(S_{t+1}) \giv[\middle]{} s_t, a_t\right] & \ (\forall 0 \leq t < T - 1) \\
  \qsoft_{\theta,T-1}(s_{T-1}, a_{T-1}) &= \theta^T \phi(s_{T-1}, a_{T-1}) &
\end{aligned}\qquad\right\}\text{ Soft VI.}
\end{equation}
\end{lemma}

\begin{proof}
We begin by setting the derivative of the Lagrangian $\Psi$ to zero and rearranging.
Through this, we find that at optimality the policy (i.e.\ primal variables) must take the form
\begin{align}\label{eqn:policy-form}
  \policy{}_t(a_t \giv s_t) =& \exp\vast(\theta^T \phi(s_t, a_t) - 1 + \frac{\mu_{s_t}}{\rho_{\policy,t}(s_t)} \\
  &+ \expect_{\policy}\left[\sum_{t'=t+1}^{T-1} \discount^{t' - t}\left(\theta^T \phi(S_{t'}, A_{t'}) - \log \policy{}_{t'}(A_{t'} \giv S_{t'})\right) \giv[\middle]{} s_t, a_t\right]\vast)\,. \nonumber
\end{align}
We abuse notation by assuming that the last term vanishes when $t=T-1$, so that at the end of the trajectory we have
\begin{equation}
    \pi_{T-1}(a_{T-1} \mid s_{T-1}) = \exp\left( \theta^T \phi(s_{T-1}, a_{T-1}) - 1 + \frac{\mu_{s_{T-1}}}{\rho_{\policy,T-1}(s_{T-1})} \right)\,.
\end{equation}

Naively calculating the optimal policy from \cref{eqn:policy-form} would require enumeration of exponentially many trajectories to obtain the inner expectation.
We will instead show that $\policy{}$ can be recovered by soft value iteration.
First we decompose $\policy_{\theta, t}(a_t \giv s_t)$ using \cref{eqn:policy-soft-decomposition}, where $\qsoft{}$ and $\vsoft{}$ functions are defined as:
\begin{align}
    &\vsoft_{\theta,t}(s_t) \triangleq 1 - \frac{\mu_{s_t}}{\rho_{\policy,t}(s_t)}\,, \\
    &\qsoft_{\theta,t}(s_t, a_t) \triangleq \theta^T \phi(s_t, a_t) + \expect_{\policy{}}\left[\sum_{t'=t+1}^{T-1} \discount^{t' - t} (\theta^T \phi(S_{t'}, A_{t'}) - \log \policy{}_{\theta, t'}(A_{t'} \giv S_{t'})) \giv[\middle]{} s_t, a_t \right]\,, \\
    &\qsoft_{\theta,T-1}(s_{T-1}, a_{T-1}) \triangleq \theta^T \phi(s_{T-1}, a_{T-1})\,.
\end{align}
We can show that $\qsoft_{\theta,t}(s_t, a_t)$ and $\vsoft_{\theta,t}(s_t)$ satisfy a ``softened'' version of the Bellman equations.
Note that by setting the $\mu_{s_t}$ variables to be appropriate normalising constants, we ensure that the normalisation KKT conditions in \cref{eqn:mce-irl-dual-func-kkt-norm-cond} are satisfied.
Since we wish to constrain ourselves to normalised policies $\pi \in \mathscr P$, we can assume the $\mu_{s_t}$ are chosen such that $\sum_{a \in \actionspace}\policy_t(a \giv s_t) = 1$ for all $s_t \in \statespace$.
It then follows that $\vsoft_{\theta,t}(s_t)$ must be a soft maximum over $\qsoft_{\theta,t}(s_t, a_t)$ values in $s_t$:
\begin{align}
 1 &= \sum_{a_t \in \actionspace} \policy{}_{\theta, t}(a_t \giv s_t) \\
 &= \sum_{a_t \in \actionspace} \exp(\qsoft_{\theta,t}(s_t, a_t) - \vsoft_{\theta,t}(s_t))\\
 1 \cdot \exp \vsoft_{\theta,t}(s_t)  &= \left( \sum_{a_t \in \actionspace} \exp(\qsoft_{\theta,t}(s_t, a_t) - \vsoft_{\theta,t}(s_t)) \right) \cdot \exp \vsoft_{\theta,t}(s_t)  \\
 &= \sum_{a_t \in \actionspace} \exp \qsoft_{\theta,t}(s_t, a_t)\\
 \vsoft_{\theta,t}(s_t) &= \log \sum_{a_t \in \actionspace} \exp \qsoft_{\theta,t}(s_t, a_t)\,.
\end{align}

Likewise, we can use the definitions of $\qsoft_{\theta,t}$ and $\vsoft_{\theta,t}$ to show that for $t < T - 1$, we have
\begin{align}
    \qsoft_{\theta,t}(s_t, a_t)
    &= \theta^T \phi(s_t, a_t) + \expect_{\policy{}_{\theta}}\left[\sum_{t'=t+1}^{T-1} \discount^{t' - t}(\theta^T \phi(S_{t'}, A_{t'}) - \log \policy{}_{\theta,t'}(A_{t'} \giv S_{t'})) \giv[\middle]{} S_t = s_t, A_t = a_t\right]\\
    &= \theta^T \phi(s_t, a_t) + \discount \expect_{\transitiondist} \vast[\expect_{\policy{}_{\theta}}\vast[\qsoft_{\theta, t+1}(S_{t+1}, A_{t+1}) - \\
    &\qquad\qquad\qquad\qquad\qquad\qquad\: \log \policy{}_{\theta,t+1}(A_{t+1} \giv S_{t+1}) \giv[\vast]{} S_{t+1} \vast] \giv[\vast]{} S_t = s_t, A_t = a_t\vast] \nonumber \\ %
    &= \theta^T \phi(s_t, a_t) + \discount \expect_{\transitiondist} \vast[\expect_{\policy{}_{\theta}}\vast[\qsoft_{\theta, t+1}(S_{t+1}, A_{t+1}) - \\
    &\qquad\qquad\qquad\qquad\qquad\: \left(\qsoft_{\theta, t+1}(S_{t+1}, A_{t+1}) - \vsoft_{\theta, t+1}(S_{t+1})\right) \giv[\vast]{} S_{t+1} \vast] \giv[\vast]{} S_t = s_t, A_t = a_t \vast] \nonumber \\ %
    &= \theta^T \phi(s_t, a_t) + \discount \expect_{\transitiondist} \left[\vsoft_{\theta, t+1}(S_{t+1}) \giv[\middle]{} S_t = s_t, A_t = a_t\right]\,,
\end{align}
where the penultimate step follows from substituting $\policy_{\theta,t}(a_t \giv s_t) = \exp\left(\qsoft_{\theta,t}(s_t, a_t) - \vsoft_{\theta,t}(s_t)\right)$.

Putting it all together, we have the soft VI equations:
\begin{align}
  \vsoft_{\theta,t}(s_t) &= \log \sum_{a_t \in \actionspace} \exp \qsoft_{\theta,t}(s_t, a_t)\\
  \qsoft_{\theta,t}(s_t, a_t) &= \theta^T \phi(s_t, a_t) + \discount \expect_{\mathcal T}\left[\vsoft_{\theta, t+1}(S_{t+1}) \giv[\middle]{} s_t, a_t\right] \\
  \qsoft_{\theta,T}(s_T, a_T) &= \theta^T \phi(s_T, a_T).
\end{align}
\end{proof}

\begin{remark}[Soft VI]
Observe that the soft VI equations are analogous to traditional value iteration, but with the hard maximum over actions replaced with a log-sum-exp, which acts as a ``soft'' maximum.
In the finite-horizon case, these equations can be applied recursively from time $t=T-1$ down to $t=0$, yielding an action distribution $\policy_{\theta,t}(a_t \giv s_t) = \exp(\qsoft_{\theta,t}(s_t, a_t) - \vsoft_{\theta,t}(s_t))$ at each time step $t$ and state $s_t$.
In the infinite-horizon case, we drop the subscripts $t$ and search for a fixed point $\vsoft_{\theta}$ and $\qsoft_{\theta}$ to the soft VI equations with corresponding stationary policy $\policy_{\theta}$.

In both cases, the agent chooses actions with probability exponential in the soft advantage $\asoft_{\theta,t}(s_t, a_t) \triangleq \qsoft_{\theta,t}(s_t, a_t) - \vsoft_{\theta,t}(s_t)$.
Thus, computing the dual function $g(\theta)$ reduces to solving a planning problem.
The similarity between the soft VI equations and the ordinary Bellman equations means that we are (somewhat) justified in interpreting the dual variable vector $\theta$ as weights for a reward function $r_{\theta}(s_t, a_t) = \theta^T \phi(s_t, a_t)$.
\end{remark}

\subsubsection{Learning a reward function with dual ascent}

\begin{algorithm}[t]
\begin{algorithmic}[1]
\State{Initialise some reward parameter estimate $\theta_0$, and set $k \gets 0$}
\Repeat{}
  \State{Apply soft value iteration to obtain optimal policy $\policy_{\theta_k}$ w.r.t $\theta_k$ (\cref{eqn:policy-soft-vi})}
  \State{$\theta_{k+1} \gets \theta_k + \alpha_k \left(\expect\limits_{\mathcal D}\left[\sum_{t=0}^{T-1} \discount^t \nabla_{\theta} r_\theta(S_t, A_t)\right] - \expect\limits_{\policy_{\theta_k}}\left[\sum_{t=0}^{T-1} \discount^t \nabla_{\theta} r_\theta(S_t, A_t)\right]\right)$ (\cref{eqn:mce-irl-dual-gradient})}
  \State{$k \gets k + 1$}
\Until{Stopping condition satisfied}
\end{algorithmic}
\caption{Maximum Causal Intropy (MCE) IRL on demonstrator $\mathcal D$}\label{alg:mce-irl}
\end{algorithm}

In the previous subsection, we determined how to find the policy $\pi$ which attains the value of the dual function $g(\theta)$, so that $g(\theta) = \Lambda(\policy{}, \theta)$.
All that remains is to determine the step that we need to take on the dual variables $\theta$ (which we are interpreting as reward parameters).
The gradient of the dual w.r.t $\theta$ is given by
\begin{equation}\label{eqn:mce-irl-dual-gradient}
  \nabla_\theta \Lambda(\policy{}, \theta) = \expect_{\policy{}}\left[\sum_{t=0}^{T-1} \discount^t \phi(S_t, A_t)\right] - \expect_{\expertdataset{}}\left[\sum_{t=0}^{T-1} \discount^t \phi(S_t, A_t)\right]\,.
\end{equation}
The first (policy $\policy$) term can be computed by applying soft VI to $\theta$ and then rolling the optimal policy forward to obtain occupancy measures.
The second (demonstrator $\expertdataset{}$) term can be computed directly from demonstration trajectories.
This cycle of performing soft VI followed by a gradient step on $\theta$ is illustrated in \cref{alg:mce-irl}.
In the next section, we'll see an alternative perspective on the same algorithm that does not appeal to feature expectation matching or duality, and which can be extended to non-linear reward functions.

\subsection{MCE IRL as maximum likelihood estimation}\label{sec:mce-irl:max-likelihood}

In the previous section, we saw that IRL can be reduced to dual ascent on a feature expectation matching problem.
We can also interpret this method as maximum likelihood estimation of reward parameters $\theta$ subject to the distribution over trajectories induced by the policy $\policy_{\theta,t}(a_t \giv s_t) = \exp(\qsoft_{\theta,t}(s_t, a_t) - \vsoft_{\theta,t}(s_t))$ under the soft VI recursion in \cref{eqn:policy-soft-vi}.
This perspective has the advantage of allowing us to replace the linear reward function $r_\theta(s_t, a_t) = \theta^T \phi(s_t, a_t)$ with a general non-linear reward function $r_\theta(s_t, a_t)$.
If $r_\theta(s_t, a_t)$ is non-linear, then the resulting problem no longer has the same maximum-entropy feature-matching justification as before, although it does still appear to work well in practice.

To show that the feature expectation matching and maximum likelihood views are equivalent when $r_\theta$ is linear, we will show that the gradient for the dual function $g(\theta)$ and the gradient for the dataset expected log likelihood are equivalent.
First, we will introduce the notion of the \textit{discounted likelihood} of a trajectory.
\begin{defn}[Discounted likelihood]
  The \textit{discounted likelihood} of a trajectory $\tau = (s_0, a_0, s_1, a_1, \ldots, a_{T-1}, s_T)$ under policy $\pi_\theta$ is
  \begin{equation}
      p_{\theta,\gamma}(\trajectory{}) = \initialstatedist(s_0) \prod_{t=0}^{T-1} \transitiondist(s_{t+1} \giv s_t, a_t)  \policy{}_{\theta,t}(a_t \giv s_t)^{\discount^t}\,,
  \end{equation}
  where $\gamma \in [0,1]$ is a discount factor.
\end{defn}
When $\gamma=1$, the discounted likelihood is simply the likelihood of $\tau$ under the policy $\pi_\theta$.
When $\gamma<1$, probabilities of actions later in the trajectory will be regularised towards 1 by raising them to $\gamma^t$.
This will not correspond to a normalised probability distribution, although we will see that it allows us to draw a connection between maximum likelihood inference and discounted MCE IRL.

Now consider the log discounted likelihood of a single trajectory $\trajectory{}$ under a model in which the agent follows the soft VI policy $\pi_{\theta,t}(a_t \giv s_t)$ with respect to reward parameters $\theta$:
\begin{align}
   \log p_\theta(\trajectory{}) &= \sum_{t=0}^{T-1} \discount^t \log \policy{}_{\theta,t}(a_t \giv s_t) + \log \initialstatedist(s_0) + \sum_{t=1}^{T} \log \transitiondist(s_{t} \giv s_{t-1}, a_{t-1}) \\
   &= \sum_{t=0}^{T-1} \discount^t \left(\qsoft_{\theta,t}(s_t, a_t) - \vsoft_{\theta,t}(s_t) \right) + f(\tau)\,.
\end{align}
Here $f(\tau)$ contains all the $\log \mathcal T$ and $\log \initialstatedist$ (dynamics) terms, which are constant in $\theta$.
Although the expression above only matches the log likelihood when $\gamma=1$, we have nevertheless added explicit discount factors so that the gradient equivalence with $g(\theta)$ holds when $\gamma < 1$.
The expected log likelihood of a demonstrator's trajectories, sampled from demonstration distribution $\expertdataset{}$, is then
\begin{align}
  \loglikelihood{}(\expertdataset{}; \theta) &= \expect_{\expertdataset{}}\left[\log p_\theta(\trajectory{})\right]\\
  &= \expect_{\expertdataset{}}\left[\sum_{t=0}^{T-1} \discount^t \left(\qsoft_{\theta,t}(S_t, A_t) - \vsoft_{\theta,t}(S_t)\right) \right] + c \\
  &= \expect_{\expertdataset{}}\Bigg[\sum_{t=0}^{T-2} \discount^t \left(r_\theta(S_t, A_t) + \discount\expect_{\transitiondist{}}\left[\vsoft_{\theta,t+1}(S_{t+1})\giv{}S_t, A_t\right] - \vsoft_{\theta,t}(S_t)\right) \\
  &\qquad\qquad + \discount^{T-1}\left(r_\theta(S_{T-1}, A_{T-1}) - \vsoft_{\theta,T-1}(S_{T-1})\right) \Bigg] + c \nonumber \\
  &= \expect_{\expertdataset{}}\left[\sum_{t=0}^{T-1} \discount^t r_\theta(S_t, A_t)\right] + \expect_{\expertdataset{}}\left[\sum_{t=0}^{T-2} \discount^{t+1} \expect_{\mathcal T}\left[\vsoft_{\theta,t+1}(S_{t+1})\giv{}S_t, A_t\right]\right] \\
   &\qquad- \expect_{\expertdataset{}}\left[\sum_{t=0}^{T-1} \discount^t \vsoft_{\theta,t}(S_t)\right] + c \nonumber.\label{eqn:mce-irl}
\end{align}
where $c$ is a constant that does not depend on $\theta$.

The transitions in the demonstration distribution $\expertdataset{}$ will follow $\transitiondist{}$ provided the demonstrator acts in the same MDP, and so we can drop the $\expectation_{\transitiondist}$. Critically, this would \emph{not} be possible were $\expertdataset{}$ instead an empirical distribution from a finite number of samples, as then $\expertdataset{}$ might not exactly follow $\transitiondist{}$. However, this is a reasonable approximation for sufficiently large demonstration datasets.

After dropping $\expectation_{\transitiondist}$, we simplify the telescoping sums:
\begin{align}
  \loglikelihood{}(\expertdataset{}; \theta) &=  \expect_{\expertdataset{}}\left[\sum_{t=0}^{T-1} \discount^t r_\theta(S_t, A_t)\right] + \expect_{\expertdataset{}}\left[\sum_{t=0}^{T-2} \discount^{t+1} \vsoft_{\theta,t+1}(S_{t+1}) - \sum_{t=0}^{T-1} \discount^t \vsoft_{\theta,t}(S_t)\right] + c \\
  &= \expect_{\expertdataset{}}\left[\sum_{t=0}^{T-1} \discount^t r_\theta(S_t, A_t)\right] - \vsoft_{\theta,0}(s_0) + c\,. \label{eqn:mce-irl-ll}
\end{align}
All that remains to show is that the gradient of this expression is equal to the gradient of the dual function $g(\theta)$ that we saw in \cref{sec:mce-irl}.

\begin{lemma}
   Assume that the reward function is linear in its parameters $\theta$, so that $r_\theta(s,a) = \theta^T \phi(s,a)$.
   Gradient ascent on the expected log likelihood from \cref{eqn:mce-irl-ll} updates parameters in the same direction as gradient descent on $g(\theta)$ from \cref{prob:mce-irl-dual}; that is,
   \begin{equation}
       \nabla_\theta \mathcal L(\mathcal D; \theta) = -\nabla_\theta g(\theta)\,,
   \end{equation}
   where the negative on the right arises from the fact that we are doing gradient \textit{ascent} on the expected log likelihood, and gradient \textit{descent} on the dual function $g(\theta)$.
\end{lemma}

\begin{proof}
We already know the gradient of $g(\theta)$ from \cref{eqn:mce-irl-dual-gradient}, so we only need to derive the gradient of $\mathcal L(\mathcal D; \theta)$.
Computing the gradient of the first term of $\mathcal L(\mathcal D; \theta)$ is trivial.
We push the differentiation operator inside the expectation so that we are averaging $\nabla r_\theta(s_t, a_t)$ over the states and actions in our dataset of samples from $\expertdataset{}$:
\begin{equation}
    \nabla_\theta \expect_{\expertdataset{}}\left[\sum_{t=0}^{T-1} \discount^t r_\theta(S_t, A_t)\right]
    = 
    \expect_{\expertdataset{}}\left[\sum_{t=0}^{T-1} \discount^t \nabla_\theta r_\theta(S_t, A_t)\right]
\end{equation}
The second term is slightly more involved, but can be derived recursively as
\begin{align}\label{eqn:v-grad}
    \nabla_\theta \vsoft_{\theta,t}(s_t) &= \nabla_\theta \log \sum_{a_t \in \mathcal A} \exp \qsoft_{\theta,t}(s_t, a_t)\\
    &= \frac{1}{\sum_{a_t \in \actionspace} \exp \qsoft_{\theta,t}(s_t, a_t)} \sum_{a_t \in \actionspace} \nabla_\theta \exp \qsoft_{\theta,t}(s_t, a_t)\\
    &= \frac{\sum_{a_t \in \actionspace} \exp \left(\qsoft_{\theta,t}(s_t, a_t)\right) \nabla_\theta \qsoft_{\theta,t}(s_t, a_t)}{\exp \vsoft_{\theta,t}(s_t)} \\
    &= \sum_{a_t \in \actionspace} \policy{}_{\theta,t}(a_t \giv s_t) \nabla_{\theta} \qsoft_{\theta,t}(s_t, a_t)\\
    &= \expect_{\policy{}_{\theta,t}}\left[\nabla_\theta r_\theta(s_t, A_t) + \discount\expect_{\transitiondist}\left[\nabla_\theta \vsoft_{\theta,t+1}(S_{t+1})\giv[\middle]{}s_t,A_t\right] \giv[\middle]{} s_t\right]\\
    &= \expect_{\policy{}_\theta}\left[\sum_{t'=t}^{T-1} \discount^{t' - t}\nabla_\theta r_\theta(S_{t'}, A_{t'}) \giv[\middle]{} s_t \right]\,,
\end{align}
where the last step unrolled the recursion to time $T$.
Armed with this derivative, the gradient of our log likelihood is
\begin{equation}
\nabla_\theta \loglikelihood{}(\expertdataset; \theta) = {\expect_{\expertdataset{}}\left[\sum_{t=0}^{T-1} \discount^t \nabla_\theta r_\theta(S_t, A_t)\right]} - {\expect_{\policy{}_\theta}\left[\sum_{t=0}^{T-1} \discount^t \nabla_\theta r_\theta(S_t, A_t)\right]}\,.\label{eqn:ll-gradient}
\end{equation}
In the special case of a linear reward $r_\theta(s, a) = \theta^T \phi(s, a)$, this gradient is equal to $-\nabla_\theta g(\theta)$ (where $\nabla_\theta g(\theta)$ is given in \cref{eqn:mce-irl-dual-gradient}).
This shows that maximising the log likelihood $\mathcal L$ is equivalent to minimising the dual function $g(\theta)$ from the previous section.
\end{proof}

\subsection{Maximum Entropy (ME) IRL: MCE IRL for deterministic MDPs}\label{sec:mce-irl:max-ent}

So far we've considered the general setting of stochastic transition dynamics, where for any given state--action pair $(s_t, a_t) \in \statespace \times \actionspace$ there may be many possible successor states $s_{t+1} \in \statespace$ for which $\transitiondist{}(s_{t+1} \giv s_t, a_t) > 0$.
Consider what happens if we act under the Maximum Causal Entropy (MCE) policy $\policy_{\theta,t}(a_t \giv s_t) = \exp(Q_{\theta,t}(s_t, a_t) - V_{\theta,t}(s_t))$ for some $\theta$, but restrict ourselves to the case of deterministic dynamics, where $\transitiondist{}(s_{t+1} \giv s_t, a_t) = 1$ for one successor state $s_{t+1}$ and zero for all others.
Here the discounted likelihood of some \textit{feasible} trajectory $\tau$---in which the state transitions agree with the dynamics---simplifies substantially.

\begin{lemma}\label{lem:me-irl}
  In an MDP with deterministic state transitions and a deterministic initial state distribution, the MCE IRL density $p_{\theta,\gamma}(\tau)$ takes the form
\begin{equation}
\labelAndRemember{eqn:me-density-form}{
    p_{\theta,\gamma}(\tau) = \begin{cases}
    \frac{1}{Z_\theta} \exp\left(\sum_{t=0}^{T-1} \discount^t r_\theta(s_t, a_t)\right) & \text{feasible }\tau \\
    0 & \text{infeasible }\tau\,,
    \end{cases}
}
\end{equation}
where $Z_\theta = \exp V_{\theta,0}(s_0)$ is a normalising constant dependent on the (deterministic) initial state $s_0$.
\end{lemma}

\begin{proof}
Assuming deterministic dynamics, the discounted likelihood for a feasible $\trajectory{}$ becomes:
\begin{align}
    p_{\theta,\gamma}(\trajectory{}) &= \initialstatedist(s_0) \prod_{t=0}^{T-1} \transitiondist{}(s_{t+1} \giv s_{t}, a_{t}) \policy{}_t(a_{t} \giv s_t)^{\discount^t} \\
    &= \initialstatedist(s_0)\prod_{t=0}^{T-1} \policy{}_t(a_{t} \giv s_t)^{\discount^t} \label{eqn:me-drop-dyn-fac} \\
    &= \initialstatedist(s_0)\exp\left(\sum_{t=0}^{T-1}\discount^t\left(Q_{\theta,t}(s_t, a_t) - V_{\theta,t}(s_t)\right)\right) \\
    &= \initialstatedist(s_0)\exp\vast(\discount^{T-1}\left(r_\theta(s_{T-1}, a_{T-1}) - V_{\theta,T-1}(s_{T-1})\right) \label{eqn:me-redundant-expect} \\
    &\qquad\qquad\qquad\  + \sum_{t=0}^{T-2} \discount^t\left(r_\theta(s_t, a_t) + \discount\expect_{\transitiondist}[V_{\theta,t+1}(S_{t+1}) \giv s_t, a_t] - V_{\theta,t}(s_t)\right) \vast) \nonumber \\
    &= \initialstatedist(s_0)\exp\vast(\discount^{T-1}\left(r_\theta(s_{T-1}, a_{T-1}) - V_{\theta,T-1}(s_{T-1})\right) \\
    &\qquad\qquad\qquad\ + \sum_{t=0}^{T-2}\discount^t\left(r_\theta(s_t, a_t) + \discount V_{\theta,t+1}(s_{t+1}) - V_{\theta,t}(s_t)\right)\vast) \nonumber \\
    &= \frac{\initialstatedist(s_0)}{\exp V_{\theta,0}(s_0)}\exp\left(\sum_{t=0}^{T-1} \discount^t r_\theta(s_t, a_t)\right).
\end{align}
We made use of determinism to drop the dynamics factors in \cref{eqn:me-drop-dyn-fac}, and to collapse the expectation over successor states that appeared in \cref{eqn:me-redundant-expect}.
If we now assume that the initial state distribution $\initialstatedist$ is deterministic (with all weight on the actual initial state $s_0$ at the beginning of the trajectory $\tau$), then we get $\mathcal I(s_0)=1$, from which the result follows.
\end{proof}

If we perform IRL under the assumption that our trajectory distribution follows the form of \cref{eqn:me-density-form}, then we recover the \textit{Maximum Entropy} IRL algorithm (abbreviated \emph{MaxEnt} or \emph{ME}) that was first proposed by Ziebart~\cite{ziebart2008maximum}.
As we will see in \cref{sec:dynamics-free}, this simplified algorithm lends itself well to approximation in environments with unknown (but deterministic) dynamics~\cite{finn2016guided,finn2016connection,fu2017learning}.

\subsubsection{ME IRL implies risk-seeking behaviour in stochastic environments}

Observe that the Maximum Entropy (ME) IRL discounted likelihood in \cref{eqn:me-density-form} takes a much simpler form than the likelihood induced by MCE IRL, which requires recursive evaluation of the soft VI equations.
Given that ME IRL has a simpler discounted likelihood, one might be tempted to use ME IRL in environments with stochastic dynamics by inserting dynamics factors to get a discounted likelihood of the form
\begin{equation}
\hat p_{\theta,\gamma}(\tau) = \initialstatedist(s_0)\prod_{t=0}^{T-1} \transitiondist{}(s_{t+1} \giv s_{t}, a_{t}) \exp\left(\discount^t r_\theta(s_t, a_t)\right)\,.
\end{equation}

Unfortunately, ``generalising'' ME IRL in this way may not produce appropriate behaviour in environments with truly non-deterministic dynamics.
The intuitive reason for this is that the ME IRL discounted likelihood in~\cref{eqn:me-density-form} can assign arbitrarily high likelihood to any feasible trajectory, so long as the return associated with that trajectory is high enough.
Indeed, by moving the dynamics factors into the $\exp$, we can see that the discounted likelihood\footnote{
    Previously, we defined the discounted likelihood as the ordinary trajectory likelihood with discount factors applied at each time step.
    In order to obtain a valid probability distribution, the discounted likelihood will have to be normalised.
    The reasoning in this section still applies after normalisation: as the return of a trajectory increases, its probability under the model will go to 1, even if stochastic transitions make it impossible to find a policy that actually achieves such a trajectory distribution.
} takes the form
\begin{equation}
\hat p_{\theta,\gamma}(\tau) = \exp\left(R_\theta(\tau) + \log \initialstatedist(s_0) + \sum_{t=0}^{T-1} \log \transitiondist{}(s_{t+1} \giv s_t, a_t)\right)\,,
\end{equation}
where $R_\theta(\tau) = \sum_{t=0}^{T-1} \discount^t r_\theta(s_t, a_t)$. This form suggests that the agent can simply pay a ``log cost'' in order to choose outcomes that would give it better return.
A trajectory $\tau$ with extremely high return $R_\theta(\tau)$ relative to other trajectories may have high likelihood even when the associated transition probabilities are low.
In reality, if a non-deterministic environment has a state transition $(s,a,s')$ with very low (but non-zero) probability $\epsilon$, then no trajectory containing that transition can have likelihood more than $\epsilon$ under any physically realisable policy, even if the trajectory has extremely high return.

\begin{figure}
    \centering
    \scalebox{1.2}{\begin{tikzpicture}[node distance = 0.5cm, thick]
    \node (0) [circ] {$s_0$};
    \node (1) [circ, above=0.25cm of 0, xshift=1cm]{$s_1$};
    \node (r) [draw=black, right=1cm of 0]{};
    \node (2) [circ, fill=green!10, right=1cm of r, yshift=0.7cm]{$s_2$};
    \node (3) [circ, fill=red!70, below=0.3cm of 2]{$s_3$};

    \node (rew-2) [right=of 2] {$r = 1.0$};
    \node (rew-3) [right=of 3] {$r = -100.0$};

    \draw[->] (0) to [bend left] (1);
    \draw[->] (1) to [bend left] (2);
    \draw[-] (0) to (r);
    \draw[->] (r) to [bend left] (2);
    \draw[->] (r) to [bend right] (3);

    \draw[->] (r) edge[bend left] node[pos=0.5,rotate=40,above]{$50\%$} (2);
    \draw[->] (r) edge[bend right] node[pos=0.5,rotate=-40,below]{$50\%$} (3);

    \draw[->] (2) edge [loop right] (2);
    \draw[->] (3) edge [loop right] (3);
\end{tikzpicture}}
    \caption{\texttt{RiskyPath}: The agent can either take a long but sure path to the goal ($s_0 \to s_1 \to s_2$), or attempt to take a shortcut ($s_0 \to s_2)$, with the risk of receiving a low reward ($s_0 \to s_3$). This example is simplified from Ziebart's thesis~\cite[Figure~6.4]{ziebart2010modeling}.}
    \label{fig:task:risky-path}
\end{figure}

This mismatch between real MDP dynamics and the ME IRL model can manifest as risk-taking behaviour in the presence of stochastic transitions.
For example, in \cref{fig:task:risky-path}, an agent can obtain $s_0 \to s_1 \to s_2$ with 100\% probability, or $s_0 \to s_2$ and $s_0 \to s_3$ with 50\% probability each.
The ME IRL transition model would wrongly believe the agent could obtain $s_0 \to s_2$ with arbitrarily high probability by paying a small cost of $\log \transitiondist(s_2 \giv s_0, \cdot) = \log \frac{1}{2} = -\log 2$.
It would therefore conclude that an agent given rewards $\reward(s_2) = 1$ and $\reward(s_3) = -100$ would favour the risky path for a sufficiently small discount factor $\discount{}$, despite the true expected return always being higher for the safe path.

\section{Dynamics-free approximations to ME IRL}\label{sec:dynamics-free}

Real-world environments are often too complex to be described in the tabular form required by MCE IRL.
To perform IRL in these settings, we can use ``dynamics-free'' approximations to MCE IRL which do not require a known world model.
In this section we describe several such approximations, all of which assume the environment is deterministic (simplifying MCE to ME IRL) and that the horizon is effectively infinite (simplifying to a stationary policy).

We begin by describing the assumptions made by these algorithms in more detail, and adapting our previous definitions to the non-tabular setting.
We then consider Guided Cost Learning (GCL), which directly approximates the gradient of the ME IRL objective using importance-weighted samples~\cite{finn2016guided}.
Finally, we will describe Adversarial IRL (AIRL)~\cite{fu2017learning}, a GAN-based approximation to ME IRL, that distinguishes between state--action pairs instead of distinguishing between trajectories.
AIRL has been shown to be successful in standard RL benchmark environments, including some MuJoCo-based control tasks~\cite{fu2017learning} and some Atari games~\cite{tucker2018inverse}.

\subsection{Notation and Assumptions}

\paragraph{Determinism} GCL is based on ME IRL, which \cref{sec:mce-irl:max-ent} showed is only well-founded for deterministic dynamics.
AIRL can be derived using MCE IRL, but many of the key results only hold for deterministic environments.
Thus we will assume that environments are deterministic throughout this section.

\paragraph{Function approximation}
Since we'll be discussing results about non-parametric and non-linear reward functions, we will overload our $\qsoft{}$ and $\vsoft{}$ notation from earlier to refer to arbitrary reward functions.
For example, given some reward function $r(s,a,s')$, we use $\qsoft_{r}(s,a)$ to denote the soft Q-function under $r$.
We'll also be making use of the \textit{soft advantage function}, defined as
\begin{align}
\asoft_{r}(s, a) &\triangleq \qsoft_{r}(s,a) - \vsoft_{r}(s)\,.
\end{align}

\subsubsection{Stationary policies and infinite-horizon MDPs}%
\label{ssec:inf-horizons-suck}

In addition to nonlinear function approximation, the algorithms which we will consider in the next few sections are designed to work with stationary (time-independent) policies.
Theoretically, the use of stationary policies can be justified by appealing to infinite horizons.
In an infinite horizon problem, the value of a given state is independent of the current time step, since there will always be an infinite number of time steps following it.
Consequently, the value function $\vsoft_{\theta,t}$ and Q-function $\qsoft_{\theta,t}$ lose their time-dependence, and can instead be recovered as the fixed point of the following recurrence:
\begin{align}
   \vsoft_{r}(s) &= \log \sum_{a \in \actionspace} \exp \qsoft_{r}(s, a)\,, \\
   \qsoft_{r}(s, a) &= \expect_{\mathcal T}\left[r(s,a,S') + \discount  \vsoft_{r}(S') \giv[\middle] s,a\right]\,.
\end{align}
This yields a stationary policy $\policy(a\giv{}s) = \exp \asoft_r(s,a)$.

The notion of an ``infinite'' horizon might sound abstruse, given that real-world demonstration trajectories $\tau \sim \expertdataset$ must be finite, but we can conceptually view them as being infinite trajectories padded with a virtual absorbing state.
This helps us model situations where the demonstrator always completes the task in finite time, after which the demonstration ends.

We can also use this padding approach to model simulated tasks, where episodes often end when some \emph{termination condition} is satisfied, such as a robot falling over or a vehicle reaching its destination.
For the purpose of IRL, we can again imagine such trajectories as ending with an infinite sequence of repeated absorbing states, each with the same learned absorbing state reward value.

Note the presence of a termination condition may simplify the IRL problem considerably when episode termination is (positively or negatively) correlated with task success.
Indeed, \citet{kostrikov2019bias} observe it is possible to get seemingly-competent behaviour in the Hopper environment by simply giving the agent a fixed positive reward at all time steps until termination, then zero afterwards.
In general, the absorbing state reward tends to be low in tasks which require the agent to stay ``alive'' and high in goal-oriented tasks where the agent needs to finish an episode quickly.

\paragraph{Implementation termination bias}
In practice, implementations of IRL algorithms often wrongly treat the episodes as being of finite but variable length, implicitly assuming that the absorbing state reward is exactly zero.
To make matters worse, certain choices of neural network architecture can \textit{force} a learned reward function to always be positive (or negative) up until termination.
This can lead to situations where the algorithm ``solves'' an environment regardless of what reward parameters it learns, or where the algorithm cannot ever solve the environment because its reward function has the wrong sign.
It is possible to fix this problem by explicitly learning an absorbing state reward, but for the purpose of comparing existing algorithms we instead recommend using fixed-horizon environments, which side-steps this problem entirely~\cite{seals2020environments}.

\paragraph{Convergence of soft VI}
Soft value iteration still converges to a unique value in the infinite horizon case, provided the discount factor $\discount{}$ is less than $1$~\citep[Appendix~A.2]{haarnoja2017energy}.
In particular, \citet{haarnoja2017energy}---adapting an earlier proof by \citet{fox2016taming}---show that soft value iteration is a contraction mapping.
It then follows from the Banach fixed point theorem that soft value iteration has a unique fixed point.

\subsection{Guided cost learning: Approximation via importance sampling}

Consider the gradient of the log discounted likelihood of a demonstration distribution $\mathcal D$ under the ME IRL model, which was previously shown in \cref{eqn:ll-gradient} to be
\begin{equation}
\nabla_\theta \loglikelihood{}(\expertdataset; \theta) = \underbrace{\expect_{\expertdataset{}}\left[\sum_{t=0}^{T-1} \discount^t \nabla_\theta r_\theta(S_t, A_t)\right]}_{\term{eqn:ll-grad-data}} - \underbrace{\expect_{\policy{}_\theta}\left[\sum_{t=0}^{T-1} \discount^t \nabla_\theta r_\theta(S_t, A_t)\right]}_{\term{eqn:ll-grad-partition}}\,.
\end{equation}
Sample-based approximation of \cref{eqn:ll-grad-data} is straightforward: we can just sample $N$ trajectories $\tau_1, \tau_2, \ldots, \tau_N \sim \mathcal D$ from the demonstration distribution $\mathcal D$, then approximate the expectation with sample mean
\begin{equation}
    \expect_{\expertdataset{}}\left[\sum_{t=0}^{T-1} \discount^t \nabla_\theta r_\theta(S_t, A_t)\right]
    \approx \frac{1}{N} \sum_{i=1}^N \sum_{t=0}^{T-1} \discount^t \nabla_\theta r_\theta(s_{t,i}, a_{t,i})\,.
\end{equation}
Approximating \cref{eqn:ll-grad-partition} is harder, since recovering $\pi_\theta$ requires planning against the estimated reward function $r_\theta$.
Although we could compute $\pi_\theta$ using maximum entropy RL~\cite{haarnoja2017energy}, doing so would be extremely expensive, and we would have to repeat the process each time we update the reward parameters $\theta$.

\paragraph{Importance sampling} Instead of learning and sampling from $\pi_\theta$, we can estimate the ME IRL gradient using importance sampling.
Importance sampling is a widely applicable technique for working with distributions that are difficult to sample from or evaluate densities for.
Imagine that we would like  to evaluate an expectation $\expect_p f(X)$, where $X$ is a random variable with density $p$ and $f : \mathcal X \to \mathbb R$ is some function on the sample space $\mathcal X$.
Say that the distribution $p$ is intractable to sample from, and we only know how to evaluate an unnormalised density $\tilde p(x) = \alpha p(x)$ with unknown scaling factor $\alpha$.
However, imagine that it is still possible to draw samples from another distribution $q$ for which we can evaluate the density $q(x)$ at any point $x \in \mathcal X$.
In this case, we can rewrite our expectation as:
\begin{align}
    \expect_p \left[f(X)\right] &= \int \! f(x) p(x) \, dx \\
    &= \int \! f(x) \frac{p(x)}{q(x)} q(x) \, dx \\
    &= \frac{1}{\alpha} \int \! f(x) \frac{\tilde p(x)}{q(x)} q(x) \, dx \\
    &= \frac{1}{\alpha} \expect_q\left[\frac{\tilde p(X)}{q(X)} f(X)\right]~.
\end{align}
Similarly, we can evaluate $\alpha$ using an expectation with respect to $q$:
\begin{align}
\alpha &= \alpha \int \! p(x) \, dx \\
&= \int \! \tilde p(x) \, dx \\
&= \int \! \frac{\tilde p(x)}{q(x)} q(x) \, dx \\
&= \expect_q \left[\frac{\tilde p(X)}{q(X)}\right]~.
\end{align}
If we let $w(x) = \tilde p(x) / q(x)$, then we have
\begin{equation}\label{eqn:is-general-formula}
    \expect_p \left[f(X)\right] = \frac{\expect_q \left[w(X) f(X)\right]}{\expect_q\left[w(X)\right]}~.
\end{equation}
Both the numerator and the denominator of \cref{eqn:is-general-formula} can be approximated by averaging $w(X) f(X)$ and $w(X)$ on samples drawn from $q$.
This removes the need to sample from $p$ or evaluate its density directly.

\paragraph{Applying importance sampling to IRL} Now consider how we could use importance sampling to estimate \cref{eqn:ll-grad-partition}.
In this case, we wish to sample from the distribution induced by $\pi_\theta$, which has discounted likelihood $p_{\theta,\gamma}(\tau) = \frac{1}{Z} \exp r_\theta(\tau)$, for $Z = \int\!\exp r_\theta(\tau)\,d\tau$ integrated over all feasible trajectories $\tau$.\footnote{
  Note that if $\gamma=1$, then $p_{\theta,\gamma}$ will be undefined over infinite horizons for some MDPs, since $r_\theta(\tau)$ will diverge. On the other hand, if $\gamma<1$, then $p_{\theta,\gamma}$ will not be a normalised density unless we incorporate an additional $\gamma$-dependent factor into the $Z(\theta)$ we derived in \cref{lem:me-irl}.
  We will ignore these subtleties for the remainder of the section by implicitly assuming that the horizon is finite, but that the policy is still stationary.
}
Again, computing $Z$ is intractable---doing so is equivalent to optimally solving an entropy-regularised RL problem.
Instead of evaluating $Z$ directly, we can use importance sampling: if there is some easy-to-compute trajectory distribution $q(\tau)$ that we \textit{can} draw samples from, then we can rewrite the expectation in \cref{eqn:ll-grad-partition} as\footnote{To avoid confusion with the horizon $T$, we are treating lowercase $\tau$ as a random variable in these expectations.}
\begin{equation}
\expect_{\policy{}_\theta}\left[\sum_{t=0}^{T-1} \discount^t \nabla_\theta r_\theta(S_t, A_t)\right]
= \frac{\expect_{\tau \sim q}\left[w_\theta(\tau) \sum_{t=0}^{T-1} \discount^t \nabla_\theta r_\theta(S_t, A_t)\right]}{\expect_{\tau \sim q} w_\theta(\tau)}\,,
\end{equation}
where $w_\theta(\tau) = \left(\exp r_\theta(\tau)\right) / q(\tau)$ denotes the importance weighting function.

The main insight of Guided Cost Learning (GCL) is that the proposal distribution $q(\tau)$ can be iteratively updated to bring it \textit{close} to the optimal proposal distribution $p_\theta(\tau)$ without having to compute $p_\theta(\tau)$ exactly~\cite{finn2016guided}.
Specifically, after each cost function update, the proposal distribution $q(\tau)$ is updated by using reinforcement learning to maximise the entropy-regularised return,
\begin{equation}
    R_H(q) = \expect_q\left[\sum_{t=0}^{T-1} \gamma^t r_\theta(S_t, A_t)\right] + H(q)~.
\end{equation}
Since GCL uses importance sampling to evaluate the ME IRL gradient, it does not need to train $q$ to optimality against $R_H(q)$.
Instead, doing a few RL steps after each cost function update can suffice to produce a good proposal distribution $q$.

In addition to the core insight that an adaptive proposal distribution can improve importance-sampled gradient estimates in ME IRL, the GCL paper also proposes a number of tricks that are necessary to make the method work in practice.
Most significantly, the actual proposal distribution $q(\tau)$ used in the GCL paper is not just a policy trained with RL, but a mixture between an RL-trained policy and an approximation of the expert trajectory distribution.
Moreover, the learnt reward function uses an architecture and pair of regularisers that make it particularly well-suited to the goal-reaching tasks that GCL was evaluated on.
In the remainder of this primer, we will introduce the Adversarial IRL algorithm~\citep{fu2017learning}, which does not rely on importance sampling and does not require as many special tricks to work well in practice.

\subsection{An interlude on generative adversarial networks}

In the next sub-section, we will see that ME IRL can be recast as a way of learning a specially-structured kind of Generative Adversarial Network (GAN).
This sub-section will briefly review the basic principles behind GANs.
Readers unfamiliar with GANs may wish to consult the original GAN paper~\cite{goodfellow2014gan}.

GANs include a generator function $x = G(z)$ that maps from a noise vector $z$ to a generated sample $x \in \mathcal X$, for some output space $\mathcal X$.
Given some fixed noise density $p_n(z)$ (e.g.\ from a unit Gaussian distribution), we define $p_g(x)$ to represent the density over $\mathcal X$ induced by sampling $Z \sim p_n$ and then computing $X = G(Z)$.
The overarching objective of GAN training is to ensure that $p_g = p_{\text{data}}$, i.e.\ to have $X$ match the true data distribution.

In GAN training, we alternate between training a \textit{discriminator} $D : \mathcal X \to [0,1]$ to distinguish between $p_g$ and $p_{\text{data}}$, and training the generator $G(z)$ to produce samples that appear ``real'' to $D(x)$.
Specifically, $D(x)$ is treated as a classifier that predicts the probability that $x$ is a sample from $p_{\text{data}}$ instead of $p_g$.
$D(x)$ is trained to minimise the cross-entropy loss
\begin{align}\label{eqn:gan-discrim-objective}
    \loss{}_D = -\expect_{Z \sim p_z}\left[ \log (1 - D(G(Z))) \right] - \expect_{X \sim p_{\text{data}}}\left[ \log D(X) \right]\,.
\end{align}
In tandem, $G$ is trained to \textit{maximise} $L_D$---that is, we want to solve $\max_G \min_D L_D$.
If $D$ can be an arbitrary function, then $\min_D L_D$ is attained at~\cite[Proposition~1]{goodfellow2014gan}:
\begin{equation}
\labelAndRemember{eq:gan-optimal-discriminator}{
\discriminator_G^*(x) = \frac{p_{\text{data}}(x)}{p_{\text{data}}(x) + p_g(x)}}\,.
\end{equation}
It is known that $G \in \argmax_G \min_D L_D$ if and only if $p_g = p_{\text{data}}$~\citep[Theorem~1]{goodfellow2014gan}. As a proof sketch, if $p_g = p_{\text{data}}$ then the loss is $\loss{}_D = \log 4$. Moreover, it can be shown that $\min_D \loss{}_D$ is equal to $\log 4 - 2 \cdot \operatorname{JSD}(p_{\text{data}} \| p_g)$, where $\operatorname{JSD}(p \| q)$ denotes the Jensen-Shannon divergence between distributions $p$ and $q$. Since the Jensen-Shannon divergence is always non-negative and becomes zero if and only if $p_{\text{data}} = p_g$, it follows that $\log 4$ is the optimal value of $\max_G \min_D L_D$, and that $p_g = p_{\text{data}}$ is the only choice of $G$ that attains this optimum.

\subsection{Adversarial IRL: A state-centric, GAN-based approximation}

GANs can be used to produce an IRL method similar to GCL, called GAN-GCL~\citep{finn2016connection}.
However, GAN-GCL obtains poor performance in practice~\cite[Table~2]{fu2017learning}.
In particular, performing discrimination over entire \emph{trajectories} $\tau$ can complicate training.

Adversarial IRL (AIRL)~\cite{fu2017learning} instead discriminates between state-action pairs.
Concretely, it learns a (stationary) stochastic policy $\policy{}(a \giv s)$ and reward function $f_\theta(s, a)$.
The AIRL discriminator is defined by:
\begin{equation}
\discriminator{}_{\theta}(s,a) = \frac{\exp f_{\theta}(s, a)}{\exp f_{\theta}(s, a) + \policy{}(a \giv s)}\,.
\end{equation}
The generator $\policy{}$ is trained with the discriminator confusion:
\begin{equation}
\labelAndRemember{eq:airl:generator-reward}{\hat{r}_{\theta}(s,a) \triangleq \log \discriminator_{\theta}(s,a) - \log\left(1 - \discriminator_{\theta}(s,a)\right)}\,.
\end{equation}
The discriminator is trained with cross-entropy loss, updating only the parameters $\theta$ corresponding to the reward function $f_{\theta}$:
\begin{equation}
\label{eq:airl:discriminator-loss}
\discriminatorloss(\theta) \triangleq -\expect_{\policy{}}\left[\sum_{t=0}^{T-1} \log\left(1 - \discriminator{}_{\theta}(S_t,A_t)\right)\right] - \expect_{\expertdataset{}} \left[\sum_{t=0}^{T-1} \log \discriminator{}_{\theta}(S_t,A_t)\right]\,.
\end{equation}

Recall we are working in the setting of \emph{infinite} horizon MDPs, since AIRL (and GCL) assume stationary policies.
Yet the summation in \cref{eq:airl:discriminator-loss} is undiscounted, and would usually diverge if $T = \infty$.
For AIRL, we must therefore additionally assume that the transition dynamics are \emph{proper}, such that the MDP will almost surely enter an absorbing state in finite time.
In \cref{eq:airl:discriminator-loss}, the trajectory fragments $\tau$ consist of the finite-length component prior to entering an absorbing state, ensuring the loss is well-defined.
Unfortunately, omitting the absorbing state from $\tau$ also means the discriminator---and therefore reward function---cannot learn an absorbing-state reward.

In the following sections, we summarise the key theoretical results of AIRL. To make it clear where assumptions are being used, we state intermediate lemmas with minimal assumptions. However, all theorems require deterministic dynamics, and the theorems in Section~\ref{sec:airl:advantage-recovery} additionally rely on the dynamics being ``decomposable''.

\subsubsection[Policy objective is entropy-regularised f]{Policy objective is entropy-regularised $f_{\theta}$}
The reward $\hat{r}_{\theta}$ used to train the generator $\policy(a \giv s)$ is a sum of two widely used GAN objectives: the minimax objective $-\log\left(1 - D_{\theta}(s,a)\right)$ and the more widely used $\log D_{\theta}(s,a)$ objective.
Combining the two objectives is reasonable since both objectives have the same fixed point.
Moreover, the combined objective has a pleasing interpretation in terms of $f_{\theta}$:
\begin{equation}
\hat{r}(s,a) = f_{\theta}(s,a) - \log \policy{}(a \giv s)\,.
\end{equation}

Summing over entire trajectories, we obtain the entropy-regularised policy objective:
\begin{equation}
\label{eq:airl:policy-objective}
\expect_{\policy{}} \left[\sum_{t=0}^{T-1} \discount^t \hat{r}(S_t, A_t)\right] = \expect_{\policy{}}\left[\sum_{t=0}^{T-1} \discount^t \left(f_{\theta}(S_t, A_t) - \log \policy{}(A_t \giv S_t)\right)\right]\,.
\end{equation}
It is known that the policy $\policy$ that maximises this objective is~\cite{ziebart2010modeling,haarnoja2017energy}:
\begin{equation}
\policy(a \giv s) = \exp\left(\qsoft_{f_\theta}(s, a) - \vsoft_{f_\theta}(s)\right) = \exp \asoft_{f_\theta}(s, a)\,.
\end{equation}

\subsubsection[f recovers the optimal advantage]{$f_{\theta}(s,a)$ recovers the optimal advantage}
\label{sec:airl:advantage-recovery}
Recall that for a fixed generator $G$, the optimal discriminator $\discriminator$ is
\begin{equation}
\recallLabel{eq:gan-optimal-discriminator},
\end{equation}
where $p_\text{data}(x)$ is the true data distribution (in our case, expert trajectories), and $p_g(x)$ is the distribution induced by the generator (in our case, policy $\policy{}$).

Normally in GAN training, it is computationally efficient to \textit{sample} from the generator distribution $p_g$ but expensive to evaluate the \textit{density} $p_g(x)$ of a given sample.
Fortunately, in AIRL (like GAN-GCL), density evaluation is cheap since the generator $p_g(x)$ is defined by a stochastic policy $\policy{}(a \giv s)$, which explicitly defines a distribution.

Suppose the generator is always pitted against the optimal discriminator in \cref{eq:gan-optimal-discriminator}.
It is known that the generator maximizes the loss in \cref{eq:airl:discriminator-loss} of the optimal discriminator if and only if $p_g(x) = p_{\text{data}}(x)$.
In our case this is attained when the generator $\policy{}(a \giv s)$ is equal to the expert policy $\policy{}_E(a \giv s)$.

The optimal discriminator for an optimal generator will always output $\frac{1}{2}$: i.e.\ $f^*_{\theta}(s, a) = \log \policy{}_E(a \giv s)$.
Moreover, the expert policy $\policy{}_E$ is the optimal maximum causal entropy policy for the MDP's reward function $r$, so $\policy{}_E(a \giv s) = \exp\asoft_r(s,a)$.
So, if the GAN converges, then at optimality $f^*_{\theta}(s, a) = \asoft_r(s, a)$ and $\policy{}(a \giv s) = \policy{}_E(a \giv s)$.

\subsubsection{Reward shaping in MCE RL}
\label{sec:airl:shaping}

In this section, we will introduce a classical result of (hard) optimal policy equivalence under potential shaping due to \citet{ng1999policy}, and then generalise it to the case of maximum entropy (soft) optimal policies.

\begin{defn}
Let $r$ and $r'$ be two reward functions. We say $r$ and $r'$ \emph{induce the same hard optimal policy} under transition dynamics $\transitiondist$ if, for all states $s \in \statespace{}$:
\begin{equation}
\arg\max_a Q_{r, \transitiondist}(s, a) = \arg\max_a Q_{r', \transitiondist}(s, a)\,.
\end{equation}
\end{defn}

\begin{theorem}
Let $r$ and $r'$ be two reward functions. $r$ and $r'$ induce the same hard optimal policy under all transition dynamics $\transitiondist$ if:
\begin{equation}
r'(s, a, s') = \lambda\left(r(s,a,s') + \discount\phi(s') - \phi(s)\right)\,,
\end{equation}
for some $\lambda > 0$ and potential-shaping function $\phi:\statespace \to \mathbb{R}$.
\end{theorem}
\begin{proof}
See \cite{ng1999policy,gleave2020quantifying}.
\end{proof}

\begin{defn}
Let $r$ and $r'$ be two reward functions. We say they \emph{induce the same soft optimal policy} under transition dynamics $\transitiondist$ if, for all states $s \in \statespace{}$ and actions $a \in \actionspace{}$:
\begin{equation}
\asoft_{r, \transitiondist}(s, a) = \asoft_{r', \transitiondist}(s, a)\,.
\end{equation}
\end{defn}

\begin{theorem}
\label{thm:soft-equivalent-shaping}
Let $r$ and $r'$ be two reward functions. $r$ and $r'$ induce the same soft optimal policy under all transition dynamics $\transitiondist$ if $r'(s, a, s') = r(s,a, s') + \discount\phi(s') - \phi(s)$ for some potential-shaping function $\phi:\statespace \to \mathbb{R}$.
\end{theorem}
\begin{proof}

We have:
\begin{equation}
\qsoft_{r',\transitiondist}(s,a) = \expect_{\transitiondist}\left[r(s,a, S') + \discount\phi(S') - \phi(s) + \discount\vsoft_{r',\transitiondist}(S') \giv[\middle]{} s, a\right]\,.
\end{equation}
So:
\begin{align}
\qsoft_{r', \transitiondist}(s,a) + \phi(s) &= \expect_{\transitiondist}\vast[r(s,a, S') + \discount\log \sum_{a' \in \actionspace} \exp\left(\qsoft_{r',\transitiondist}(S', a') + \phi(S')\right) \giv[\vast]{} s, a\vast]\,.
\end{align}
Thus $\qsoft_{r',\transitiondist}(s,a) + \phi(s)$ satisfies the soft Bellman backup for $r$, so:
\begin{equation}
\qsoft_{r, \transitiondist}(s, a) = \qsoft_{r', \transitiondist}(s,a) + \phi(s)\,.
\end{equation}
It follows that the optimal advantage is invariant to shaping:
\begin{align}
\asoft_{r', \transitiondist}(s, a) &= \qsoft_{r', \transitiondist}(s, a) - \vsoft_{r', \transitiondist}(s) \\
&= \qsoft_{r', \transitiondist}(s, a) - \log \sum_{a \in \actionspace} \exp\left(\qsoft_{r', \transitiondist}(s, a)\right) \\
&= \qsoft_{r, \transitiondist}(s, a) + \phi(s) - \log \sum_{a \in \actionspace} \exp\left(\qsoft_{r, \transitiondist}(s, a) + \phi(s)\right) \\
&= \qsoft_{r, \transitiondist}(s, a) - \log \sum_{a \in \actionspace} \exp\left(\qsoft_{r, \transitiondist}(s, a)\right) \\
&= \asoft_{r, \transitiondist}(s, a)\,.
\end{align}
\end{proof}

\begin{remark}
Note that rescaling $r$ by $\lambda \neq 1$ changes the soft optimal advantage function and, therefore, the soft-optimal policy. In particular, it approximately rescales the soft optimal advantage function (this is not exact as log-sum-exp is non-linear). Rescaling will therefore tend to have the effect of making the soft-optimal policy less ($\lambda > 1)$ or more ($\lambda < 1$) stochastic.
\end{remark}

\subsubsection{Discriminator objective}
In this section, we show that minimising the loss of the discriminator corresponds to ME IRL in deterministic dynamics when $f_{\theta}$ is already an advantage for some reward function.

\begin{theorem}
Consider an undiscounted, deterministic MDP. Suppose $f_{\theta}$ and $\policy_{\theta}$ are the soft-optimal advantage function and policy for reward function $r_{\theta}$. Then minimising the cross-entropy loss of the discriminator under generator $\policy_{\theta}$ is equivalent to maximising the log-likelihood of observations under the Maximum Entropy (ME) IRL model.

Specifically, recall that the gradient of ME IRL (with $\discount = 1$) log likelihood is
\begin{equation*}
\nabla_\theta \loglikelihood{}(\expertdataset; \theta) = {\expect_{\expertdataset{}}\left[\sum_{t=0}^{T-1} \nabla_\theta r_\theta(S_t, A_t)\right]} - {\expect_{\policy{}_\theta}\left[\sum_{t=0}^{T-1} \nabla_\theta r_\theta(S_t, A_t)\right]}\,.
\end{equation*}
We will show that the gradient of the discriminator objective is
\begin{equation}
-2 \nabla_{\theta} \discriminatorloss(\theta) = \expect_{\expertdataset}\left[\sum_{t=0}^{T-1} \nabla_{\theta} f_{\theta}(S_t, A_t)\right] - \expect_{\policy{}_{\theta}}\left[\sum_{t=0}^{T-1} \nabla_{\theta} f_{\theta}(S_t, A_t)\right]\,.
\end{equation}
\end{theorem}

\begin{proof}
Note that $\discriminatorloss(\theta)$ is the discriminator loss, and so we wish to maximise the discriminator objective $-\discriminatorloss(\theta)$. This has the gradient:
\begin{equation}
-\nabla_{\theta} \discriminatorloss(\theta) = \expect_{\expertdataset{}} \left[\sum_{t=0}^{T-1} \nabla_{\theta} \log \discriminator{}_{\theta}(S_t,A_t,S_{t+1})\right] + \expect_{\policy{}}\left[\sum_{t=0}^{T-1} \nabla_{\theta} \log\left(1 - \discriminator{}_{\theta}(S_t,A_t,S_{t+1})\right)\right] \,.
\end{equation}
Observe that:
\begin{equation}
\log \discriminator_{\theta}(s_t, a_t, s_{t+1}) = f_{\theta}(s_t, a_t) - \log \left(\exp f_{\theta}(s_t, a_t) + \policy{}(a_t \giv s_t)\right)\,.
\end{equation}
Thus:
\begin{equation}
\nabla_{\theta} \log \discriminator_{\theta}(s_t, a_t) = \nabla_{\theta} f_{\theta}(s_t, a_t) - \frac{\exp\left(f_{\theta}(s_t, a_t)\right) \nabla_{\theta} f_{\theta}(s_t, a_t)}{\exp f_{\theta}(s_t, a_t) + \policy{}(a_t \giv s_t)}\,.\label{eqn:airl-logd-gradient}
\end{equation}
Similarly:
\begin{equation}
\log \left(1 - \discriminator_{\theta}(s_t, a_t)\right) = \policy(a_t \giv s_t) - \log \left(\exp f_{\theta}(s_t, a_t) + \policy{}(a_t \giv s_t)\right)\,.
\end{equation}
So:
\begin{equation}
\nabla_{\theta} \log \left(1 - \discriminator_{\theta}(s_t, a_t)\right)) = -\frac{\exp\left(f_{\theta}(s_t, a_t)\right) \nabla_{\theta} f_{\theta}(s_t, a_t)}{\exp f_{\theta}(s_t, a_t) + \policy{}(a_t \giv s_t)}\,.\label{eqn:airl-logminusd-gradient}
\end{equation}

Recall we train the policy $\policy{}(a_t \giv s_t)$ to maximise \cref{eq:airl:policy-objective}.
The optimal maximum entropy policy for a given $f_{\theta}$ is $\policy^*_{f_{\theta}}(a_t \giv s_t) = \exp \asoft_{f_{\theta}}(s_t, a_t)$.

By assumption, $f_{\theta}$ is the advantage for reward function $r_{\theta}$, so:
\begin{align}
f_{\theta}(s, a) &= \asoft_{r_{\theta}}(s, a) \\
&= \qsoft_{r_{\theta}}(s, a) - \vsoft_{r_{\theta}}(s) \\
&= \sum_{s' \in \statespace} \transitiondist(s,a,s')\left(r_{\theta}(s,a,s') + \gamma \vsoft_r(s')\right) - \vsoft_r(s) \\
&= r_{\theta}\left(s,a, \transitiondist(s,a)\right) + \gamma \vsoft_r\left(\transitiondist(s,a)\right) - \vsoft_r(s)\,,
\end{align}
where we write $s' = \transitiondist(s,a)$ for the deterministic next-state. Restricting ourselves only to feasible transitions $(s,a,s')$, we can alternatively write:
\begin{equation}
f_{\theta}(s,a,s') = r(s,a,s') + \gamma \vsoft_r(s') - \vsoft_r(s)\,.
\end{equation}
That is, $f_{\theta}$ is $r_{\theta}$ shaped by potential function $\vsoft_r(s)$. Applying \cref{thm:soft-equivalent-shaping}, it follows that:
\begin{equation}
\asoft_{f_{\theta}}(s,a) = \asoft_{r_{\theta}}(s,a)\,
\end{equation}
but by assumption $f_{\theta}(s, a) = \asoft_{r_{\theta}}(s, a)$, so we have that $f_{\theta}(s,a)$ is idempotent under the advantage operator:
\begin{equation}
\asoft_{f_{\theta}}(s,a) = f_{\theta}(s,a)\,.
\end{equation}
Thus the optimal policy is $\policy^*_{f_{\theta}}(a_t \giv s_t) = \exp f_{\theta}(s_t, a_t) = \policy^*_{r_\theta}(a_t \giv s_t)$.
Substituting this expression into \cref{eqn:airl-logd-gradient} gives
\begin{equation}
\nabla_{\theta} \log \discriminator_{\theta}(s_t, a_t) = \frac{1}{2}\nabla_{\theta} f_{\theta}(s_t, a_t)\,,
\end{equation}
and into \cref{eqn:airl-logminusd-gradient} gives
\begin{equation}
\nabla_{\theta} \log \left(1 - \discriminator_{\theta}(s_t, a_t)\right)) = -\frac{1}{2}\nabla_{\theta} f_{\theta}(s_t, a_t)\,.
\end{equation}

So:
\begin{align}
-2 \nabla_{\theta} \discriminatorloss(\theta) &= \expect_{\expertdataset}\left[\sum_{t=0}^T \nabla_{\theta} f_{\theta}(S_t, A_t)\right] - \expect_{\policy{}_{\theta}}\left[\sum_{t=0}^T \nabla_{\theta} f_{\theta}(S_t, A_t)\right]\,.
\end{align}
\end{proof}

\begin{remark}
\cref{sec:airl:advantage-recovery} showed the globally optimal $f_{\theta}$ is the optimal soft advantage function. 
However, there is no guarantee that $f_{\theta}$ is ever a soft advantage function during training.
So this theorem does not demonstrate convergence, but does provide intuition for why AIRL often works well in practice.
\end{remark}

\subsubsection{Recovering rewards}

In \cref{sec:airl:shaping}, we saw that if a reward function $r'$ is a potential shaped version of $r$ then $r'$ induces the same soft $Q$-values as $r$ up to a state-only function.
In the case that both reward functions are state-only, i.e.\ $r(s)$ and $r'(s)$, then potential shaping (when $\discount < 1)$ reduces to the special case of $r'(s) = r(s) + k$ for some constant $k$.
Perhaps surprisingly, AIRL can determine state-only rewards up to a constant provided the (deterministic) transition dynamics $\transitiondist$ satisfies a strong requirement known as the \emph{decomposability condition}.

\begin{defn}[Decomposability Condition]
We define two states $u, v \in \statespace$ as being \emph{$1$-step linked} under a transition distribution $\transitiondist(s' \giv s,a)$ if there exists a state $s \in \statespace$ and actions $a, b \in \actionspace$ such that $u$ and $v$ are successor states to $s$: i.e. $\transitiondist(u \giv s, a) > 0$ and $\transitiondist(v \giv s, b) > 0$. 

We define two states $u, v \in \statespace$ as being $n+1$-step \emph{linked} if they are $n$-step linked or if there is an intermediate state $s \in \statespace$ such that $u$ is $n$-step linked to $s$ and $s$ is $1$-step linked to $v$.

We define two states $u, v \in \statespace$ as being \emph{linked} if there is some $n \in \mathbb{N}$ for which they are $n$-step linked.

A transition distribution $\transitiondist$ is \emph{decomposable} if all pairs of states in the MDP are linked.
\end{defn}

The decomposability condition can be counterintuitive, so we consider some examples before using the definition further. Note that although the decomposability condition is stated in terms of transition probabilities, later applications of this condition will assume deterministic dynamics where probabilities are either $0$ or $1$.

A simple MDP that does \textbf{not} satisfy the condition is a two-state cyclic MDP, where it is only possible to transition from state $A$ to $B$ and vice-versa. 
There is no state that can reach both $A$ and $B$, so they are not $1$-step linked.
They are therefore also not $n$-step linked for any $n$, since there are no possible intermediary states.
However, the MDP would be decomposable if the dynamics were extended to allow self-transitions (from $A \to A$ and $B \to B$).

A similar pattern holds in gridworlds. 
Imagine a checkerboard pattern on the grid.
If all actions move to an adjacent cell (left, right, up or down), then all the successors of white cells are black, and vice-versa.
Consequently, cells are only ever $1$-step linked to cells of the same colour.
Taking the transitive closure, all cells of the same colour are linked together, but never to cells of a different colour.
However, if ones adds a `stay' action to the gridworld then all cells are linked.

We have not been able to determine whether the decomposability condition is satisfied in standard RL benchmarks, such as MuJoCo tasks or Atari games.

We will now show a key application of the decomposability condition: that equality of soft optimal policies implies equality of state-only rewards up to a constant.

\begin{theorem}
Let $\transitiondist$ be a \textbf{deterministic} dynamics model satisfying the decomposability condition, and let $\discount > 0$. Let $r(s)$ and $r'(s)$ be two reward functions producing the same MCE policy in $\transitiondist$. That is, for all states $s \in \statespace$ and actions $a \in \actionspace$:
\begin{equation}
\label{eq:airl:decomposable:advantage-equal}
\asoft_{r',\transitiondist}(s,a) = \asoft_{r,\transitiondist}(s,a)\,.
\end{equation}
Then $r'(s) = r(s) + k$, for some constant $k$.
\end{theorem}
\begin{proof}
We start by considering the general case of a stochastic dynamics model $\transitiondist$ and reward functions over $(s,a,s')$ triples. We introduce the simplifying assumptions only when necessary, to highlight why we make these assumptions.

Substituting the definition for $\asoft$ in \cref{eq:airl:decomposable:advantage-equal}:
\begin{equation}
\qsoft_{r',\transitiondist}(s,a) - \vsoft_{r',\transitiondist}(s) = \qsoft_{r,\transitiondist}(s,a) - \vsoft_{r,\transitiondist}(s)\,.
\end{equation}
So:
\begin{equation}
\qsoft_{r',\transitiondist}(s,a) = \qsoft_{r,\transitiondist}(s,a) - f(s)\,,
\end{equation}
where $f(s) = \vsoft_{r,\transitiondist}(s) - \vsoft_{r',\transitiondist}(s)$. Now:
\begin{align}
\qsoft_{r',\transitiondist}&(s,a) = \qsoft_{r,\transitiondist}(s,a) - f(s) \nonumber \\
&= \expect_{\transitiondist}\left[r(s, a, S') - f(s) + \discount \log \sum_{a'} \exp\qsoft_{r,\transitiondist}(S', a') \giv[\middle]{} s, a\right] \\
&= \expect_{\transitiondist}\vast[r(s, a, S') - f(s) + \discount f(S') + \discount \log \sum_{a'} \exp\left(\qsoft_{r,\transitiondist}(S', a') - f(S')\right) \giv[\vast]{} s, a\vast] \\
&= \expect_{\transitiondist}\left[r(s, a, S') - f(s) + \discount f(S') + \discount \log \sum_{a'} \exp\qsoft_{r',\transitiondist}(S', a') \giv[\middle]{} s, a\right]\,.
\end{align}

Contrast this with the Bellman backup on $r'$:
\begin{equation}
\qsoft_{r',\transitiondist}(s,a) = \expect_{\transitiondist}\left[r'(s, a, S') + \discount \log \sum_{a'} \exp\qsoft_{r',\transitiondist}(S', a') \giv[\middle]{} s, a\right]\,.
\end{equation}
So, equating these two expressions for $\qsoft_{r',\transitiondist}(s,a)$:
\begin{equation}
\expect_{\transitiondist}\left[r'(s, a, S') \giv[\middle]{} s, a\right] = \expect_{\transitiondist}\left[r(s, a, S') - f(s) + \discount f(S') \giv[\middle]{} s, a \right]\,.
\end{equation}

In the special case of deterministic dynamics, then if $s' = \transitiondist(s, a)$, we have:
\begin{equation}
r'(s, a, s') = r(s, a, s') - f(s) + \discount f(s')\,.
\end{equation}
This looks like $r'$ being a potential-shaped version of $r$, but note this equality may not hold for transitions that are not feasible under this dynamics model $\transitiondist$.

If we now constrain $r'$ and $r$ to be state-only, we get:
\begin{equation}
r'(s) - r(s) + f(s) = \discount f(s')\,.
\end{equation}
In particular, since $\discount \neq 0$ this implies that for a given state $s$, all possible successor states $s'$ (reached via different actions) must have the same value $f(s')$. In other words, all $1$-step linked states have the same $f$ value. Moreover, as $2$-step linked states are linked via a $1$-step linked state and equality is transitive, they must also have the same $f$ values. By induction, all linked states must have the same $f$ value. Since by assumption $\transitiondist$ is decomposable, then $f(s) = c$ for some constant, and so:
\begin{equation}
r'(s) = r(s) + (\discount - 1)c = r(s) + k\,.
\end{equation}
\end{proof}

Now, we will see how the decomposability condition allows us to make inferences about equality between state-only functions. This will then allow us to prove that AIRL can recover state-only reward functions.

\begin{lemma}
\label{lemma:airl:decomposable}
Suppose the transition distribution $\transitiondist$ is decomposable. Let $a(s), b(s), c(s), d(s)$ be functions of the state. Suppose that for all states $s \in \statespace$, actions $a \in \actionspace$ and successor states $s' \in \statespace$ for which $\transitiondist(s' \giv s,a) > 0$, we have
\begin{equation}
\label{eq:airl:lemma-decomposable-assumption}
a(s) - c(s) = b(s') - d(s')\,.
\end{equation}
Then for all $s \in \statespace$,
\begin{align}
a(s) &= c(s) + k \\
b(s) &= d(s) + k\,,
\end{align}
where $k \in \mathbb{R}$ is a constant.
\end{lemma}
\begin{proof}
Let $f(s) = a(s) - c(s)$ and $g(s') = b(s') - d(s')$. 

\textbf{Base case}: Since $f(s) = g(s')$ for any successor state $s'$ of $s$, it must be that $g(s')$ takes on the same value for all successor states $s'$ for $s$. This shows that all $1$-step linked states have the same $g(s')$.

\textbf{Inductive case}: Moreover, this extends by transitivity. Suppose that all $n$-step linked states $s'$ have the same $g(s')$. Let $u$ and $v$ be $(n+1)$-step linked.
So $u$ is $n$-step linked to some intermediate $s \in \statespace$ that is $1$-step linked to $v$. But then $g(u) = g(s) = g(v)$. So in fact all $(n+1)$-step linked states $s'$ have the same $g(s')$.

By induction, it follows that all linked states $s'$ have the same $g(s')$. In a decomposable MDP, all states are linked, so $g(s)$ is equal to some constant $k \in \mathbb{R}$ for all $s \in \statespace$.\footnote{Note that in a decomposable MDP all states are the successor of at least one state.} Moreover, in any (infinite-horizon) MDP any state $s \in \statespace$ must have at least one successor state $s' \in \statespace$ (possibly itself). By assumption, $f(s) = g(s')$, so $f(s) = k$ for all $s \in \statespace$.
\end{proof}

Finally, we can use the preceding result to show when AIRL can recover a state-only reward (up to a constant). Note that even if the ground-truth reward $r(s)$ is state-only, the reward \emph{network} $f$ in AIRL must be a function of states \emph{and} actions (or next states) in order to be able to represent the global minimum: the soft advantage function $\asoft(s,a)$. The key idea in the following theorem is to decompose $f$ into a state-only reward $g_\theta(s)$ and a potential shaping term. This gives $f$ capacity to represent $\asoft(s,a)$, while the structure ensures $g_{\theta}(s)$ equals the ground-truth reward up to a constant.

\begin{theorem}
Suppose the reward network is parameterized by
\begin{equation}
f_{\theta,\phi}(s, a, s') = g_{\theta}(s) + \discount h_{\phi}(s') - h_{\phi}(s)\,.
\end{equation}
Suppose the ground-truth reward is state-only, $r(s)$. Suppose moreover that the MDP has deterministic dynamics $\transitiondist$ satisfying the decomposability condition, and that $\discount > 0$. Then if $f_{\theta^*,\phi^*}$ is the reward network for a global optimum $(D_{\theta^*,\phi^*}, \pi)$ of the AIRL problem, that is $f_{\theta^*,\phi^*}(s,a,s') = f^*(s, a, s')$, we have:
\begin{equation}
g_{\theta^*}(s) = r(s) + k_1, \quad h_{\phi^*}(s) = \vsoft_r(s) + k_2\,,
\end{equation}
where $k_1, k_2 \in \mathbb{R}$ are constants and $s \in \statespace$ is a state.
\end{theorem}
\begin{proof}
We know the global minimum $f^*(s, a, s') = \asoft(s, a)$, from \cref{sec:airl:advantage-recovery}. Now:
\begin{align}
\asoft(s, a) &= \qsoft_r(s, a) - \vsoft_r(s) \\
&= r(s) + \discount \vsoft_r(s') - \vsoft_r(s)\,.
\end{align}
So for all states $s \in \statespace$, actions $a \in \actionspace$ and resulting deterministic successor states $s' = \transitiondist(s, a)$ we have:
\begin{align}
&g_{\theta^*}(s) + \discount h_{\phi^*}(s') - h_{\phi^*}(s) = r(s) + \discount \vsoft_r(s') - \vsoft_r(s) \\
\iff &\left(g_{\theta^*}(s) - h_{\phi^*}(s)\right) - \left(r(s) - \vsoft_r(s)\right) = \discount \vsoft_r(s') - \discount h_{\phi^*}(s')\,.
\end{align}
Now applying \cref{lemma:airl:decomposable} with $a(s) = g_{\theta^*}(s) - h_{\phi^*}(s)$, $c(s) = r(s) - \vsoft_r(s)$, $b(s') = \discount \vsoft_r(s')$ and $d(s') = \discount h_{\phi^*}(s')$ gives:
\begin{align}
g_{\theta^*}(s) - h_{\phi^*}(s) &= a(s) = c(s) + k_0 = r(s) - \vsoft_r(s) + k_0 \\
\discount h_{\phi^*}(s') &= d(s') = b(s') + k_0 = \discount \vsoft_r(s') + k_0\,,
\end{align}
where $k_0 \in \mathbb{R}$ is a constant. Rearranging and using $\discount \neq 0$ we have:
\begin{align}
\label{eq:airl:g-equals-shapedr}
g_{\theta^*}(s) &= r(s) + h_{\phi^*}(s) - \vsoft_r(s) + k_0 \\
\label{eq:airl:h-equals-vsoft}
h_{\phi^*}(s') &= \vsoft_r(s') + \frac{k_0}{\discount} = \vsoft_r(s') + k_2\,.
\end{align}
By the decomposability condition, all states $s \in \statespace$ are the successor of some state (possibly themselves). So we can apply \cref{eq:airl:h-equals-vsoft} to all $s \in \statespace$:
\begin{equation}
\label{eq:airl:h-equals-vsoft-general}
h_{\phi^*}(s) = \vsoft_r(s) + k_2\,.    
\end{equation}
Finally, applying \cref{eq:airl:h-equals-vsoft-general} to \cref{eq:airl:g-equals-shapedr} yields:
\begin{equation}
g_{\theta^*}(s) = r(s) + (k_0 + k_2) = r(s) + k_1\,,
\end{equation}
as required.
\end{proof}

\begin{remark}
Note this theorem makes several strong assumptions. In particular, it requires that $f$ attains the global minimum, but AIRL is not in general guaranteed to converge to a global optimum. Additionally, many environments have stochastic dynamics or are not 1-step linked.

Note that in stochastic dynamics there may not exist any function $f(s, s')$ that is always equal to $\asoft(s, a)$. This is because there may exist $a, a'$ such that $\transitiondist(\cdot \giv s, a)$ and $\transitiondist(\cdot \giv s, a')$ differ but both have support for $s'$ while $\asoft(s, a) \neq \asoft(s, a')$ .
\end{remark}

\section{Conclusion}
We have described three Inverse Reinforcement Learning (IRL) algorithms: the tabular method Maximum Causal Entropy (MCE) IRL; and deep-learning based, dynamics-free algorithms Guided Cost Learning (GCL) and Adversarial IRL (AIRL).
We have shown that MCE IRL can be derived from maximising entropy under a feature expectation matching constraint.
Furthermore, we have shown how this is equivalent to maximising the likelihood of the data.
Finally, we have explained how GCL and AIRL can both be viewed as extensions of this maximum likelihood solution to settings with unknown dynamics and potentially continuous state and action spaces.

While contemporary methods such as GCL and AIRL have many points in common with MCE IRL, the connection is far from exact.
For example, the discriminator objective of AIRL only aligns with that of MCE IRL in undiscounted MDPs, yet it is routinely applied to discounted MDPs.
One promising direction for future work would be to derive a dynamics-free algorithm using function approximation directly from the original MCE IRL approach.
We consider it probable that such an algorithm would provide more stable performance than existing heuristic approximations to MCE IRL.

\subsection*{Acknowledgements}
We would like to thank Alyssa Li Dayan, Michael Dennis, Yawen Duan, Daniel Filan, Erik Jenner, Niklas Lauffer and Cody Wild for feedback on earlier versions of this manuscript.

\bibliographystyle{plainnat}
\bibliography{citations}{}

\end{document}